%% file: local-MQ-arxiv.tex
\documentclass{article}

\usepackage{amsmath}
\usepackage{amsthm}
\usepackage{amssymb}
\usepackage{fullpage}
\usepackage{natbib}
\usepackage{graphicx}
\usepackage{hyperref}
\usepackage{algorithm2e}
\usepackage{url}
\input{math_macros}

\title{Learning using Local Membership Queries}

\author{Pranjal Awasthi \\ Carnegie Mellon University \\
\texttt{pawasthi@cs.cmu.edu} \and Vitaly Feldman \\ IBM Almaden Research Center
\\ \texttt{vitaly@post.harvard.edu} \and Varun Kanade\thanks{The author is supported by a Simons Postdoctoral
 Fellowship.  Part of this work was performed while the author was at Harvard
 University supported by grant NSF-CCF-09-64401} \\ University of California,
 Berkeley \\ \texttt{vkanade@eecs.berkeley.edu}
}

\begin{document}
\maketitle

\begin{abstract}
\input{abstract}

\end{abstract}

%\begin{keywords}
%PAC learning, membership queries, decision trees, DNF
%\end{keywords}

\newpage

\section{Introduction}
\label{sctn:intro}
%\label{sec:intro}

\input{intro.vitaly}

\section{Notation and Preliminaries}
\label{sec:notation}
\input{setting}

\section{Learning Sparse Polynomials under Locally Smooth Distributions}
\label{sec:learn_poly}
\input{learn_poly_reals}

\section{Learning DNF Formulas under the Uniform Distribution}
\label{sec:DNF_unif}
\input{DNF_unif}

\section{Lower Bound for Agnostic Learning}
\label{sec:lower-bound}
\input{lower_bound}

\section{Conclusions and Future Work}
\label{sec:conclusion}
\input{conclusion}

%\acks{ The authors would like to thank Avrim Blum, Adam Kalai,
%Shang Hua Teng, and Leslie Valiant for several helpful discussion.}

\bibliography{allrefs}
\bibliographystyle{plainnat}

\appendix

\section{Definitions}
\label{app:notation}
\input{setting_app}

\section{Learning Sparse Polynomials under Locally Smooth Distributions}
\label{app:learn_poly}
\input{learn_poly_app}

\section{Learning DNF Formulas under the Uniform Distribution}
\label{app:DNF_unif}
\input{DNF_unif_app}

\section{Learning Decision Trees}
\label{sec:dts}
\input{dt_intro}

\subsection{Learning $\log$-Depth Trees under Locally Smooth Distributions}
\label{sec:low-degree}
\input{low_degree}

\subsection{Learning Decision Trees under the Uniform Distribution}
\label{sec:DT_unif}
\input{DT_unif}

\subsection{Learning Decision Trees under Product Distributions}
\label{sec:DT_prod}
\input{DT_prod}

\subsection{Learning under Random Classification Noise}
\label{sec:rcn}
\input{rcn}

%\section{Learning under Random Classification Noise}
%\label{app:DT_rcn}
\input{rcn_app}

\section{Lower Bound for Agnostically Learning Parities}
\label{sec:lower-bound-app}
\input{lower_bound_app}

\section{Separation Results}
\label{sec:separation}
\input{PAC_separate}

\end{document}

%% file: math_macros.tex
\newtheorem{theorem}{Theorem}[section]

\newtheorem{lemma}[theorem]{Lemma}
\newtheorem{claim}[theorem]{Claim}

\newtheorem{definition}[theorem]{Definition}
\newtheorem{fact}[theorem]{\bf Fact}

\def\th{^{\footnotesize{th}}}
\def\reals{\mathbb{R}}
\def\rationals{\mathbb{Q}}
\def\bzo{\{ b_0, b_1 \}}
\def\A{{\mathcal A}}
\newcommand{\mc}[1]{\ensuremath {{\mathcal #1}}}

\newcommand{\err}{\mathrm{err}}
\newcommand{\poly}{\mathrm{poly}}

\newcommand{\eps}{\epsilon}

\newcommand{\D}{{\cal D}}

\newcommand{\eat}[1]{ }

\def\EX{\mathsf{EX}}
\def\MQ{\mathsf{MQ}}
\def\moo{\{-1, 1\}}
\def\zo{\{0, 1\}}

\def\DC{{\mathcal D}}
\def\LA{{\mathcal L}}
\def\PACplusMQ{\mathsf{PAC\mbox{+}MQ}}
\def\PAC{\mathsf{PAC}}
\def\DNF{\mathsf{DNF}}

\def\DP{{\mathcal P}}
\def\E{\mathbb{E}}
\def\sign{\mathrm{sign}}
\def\FF{\mathbb{F}}
\def\ie{{\it i.e.}~}
\def\binomial{\mathsf{Bin}}

%% file: abstract.tex
We introduce a new model of membership query ($\MQ$) learning, where the learning
algorithm is restricted to query points that are \emph{close} to random examples
drawn from the underlying distribution. The learning model is intermediate
between the $\PAC$ model (Valiant, 1984) and the $\PACplusMQ$ model (where the
queries are allowed to be arbitrary points). 

Membership query algorithms are not popular among machine learning
practitioners.  Apart from the obvious difficulty of adaptively querying
labelers, it has also been observed that
querying \emph{unnatural} points leads to increased noise from human labelers
(Lang and Baum, 1992).  This motivates our study of learning algorithms that
make queries that are close to examples generated from the data distribution. 

%We restrict our attention to functions defined on the $n$-dimensional Boolean hypercube and say
%that a membership query is local if its Hamming distance from some example in the (random)
%training data is at most $O(\log(n))$. We show three positive learning results
%in this model: \\
%\mbox{~~}(i) The class of $O(\log(n))$-depth decision trees is learnable under a
%large class of \emph{smooth} distributions using $O(\log(n))$-local queries. \\
%\mbox{~~}(ii) The class of polynomial-sized decision trees is learnable under product
%distributions using $O(\log(n))$-local queries.\\
%\mbox{~~}(iii) The class of sparse polynomials (with coefficients in $\reals$)
%over $\{0,1\}^n$ is learnable under smooth distributions using
%$O(\log(n))$-local queries.

We restrict our attention to functions defined on the $n$-dimensional Boolean hypercube and say
that a membership query is local if its Hamming distance from some example in the (random)
training data is at most $O(\log(n))$. We show the following results
in this model: \\
\mbox{~~}(i) The class of  sparse polynomials (with coefficients in $\reals$)
over $\{0,1\}^n$ is polynomial time learnable under a large class of
\emph{locally smooth} distributions using $O(\log(n))$-local queries. This class
also includes the class of $O(\log(n))$-depth decision trees.\\
\mbox{~~}(ii) The class of polynomial-sized decision trees is polynomial time learnable under product
distributions using $O(\log(n))$-local queries.\\
\mbox{~~}(iii) The class of polynomial size $\DNF$ formulas is learnable under the uniform distribution using
$O(\log(n))$-local queries in time $n^{O(\log(\log(n)))}$.\\
\mbox{~~}(iv) In addition we prove a number of results relating the proposed
model to the traditional $\PAC$ model and the $\PACplusMQ$ model.

%% file: intro.vitaly.tex
%\label{sctn:intro}
 Valiant's Probably Approximately Correct~($\PAC$) model \citep{V84} has been used widely to study computational complexity of learning.
In the $\PAC$ model, the goal is to design algorithms which can ``learn" an
unknown target function, $f$, from a concept class, $C$ (for example, $C$ may be
polynomial-size decision trees or linear separators), where $f$ is a boolean
function over some instance space, $X$ (typically $X = \{-1, 1\}^n$ or $X
\subseteq \reals^n$).  The learning algorithm has access to random
\emph{labeled} examples, $(x, f(x))$, through an oracle, $\EX(f, D)$, where $f$
is the unknown target concept and $D$ is the target distribution. The goal of
the learning algorithm is to output a hypothesis, $h$, with low error with
respect to the target concept, $f$, under distribution, $D$.

Several interesting concept classes have been shown to be \emph{learnable} in
the $\PAC$ framework (e.g. boolean conjunctions and disjunctions,
$k$-$\mathsf{CNF}$ and
$k$-$\DNF$ formulas (for constant $k$), decision lists and the class of linear
separators).  On the other hand, it is known that very rich concept classes such
as polynomial-sized circuits are not $\PAC$-learnable under cryptographic
assumptions~\citep{V84, GGM86}.  The most interesting classes for which both
efficient $\PAC$ learning algorithms and cryptographic lower bounds have remained
elusive are polynomial-size decision trees (even $\log$-depth decision trees)
and polynomial-size $\DNF$ formulas. \medskip

\noindent {\bf Membership Query Model}: This learning setting is an extension of
the $\PAC$ model and allows the learning algorithm to query the label of any
point $x$ of its choice in the domain.  These queries are called
\emph{membership queries}.  With this additional power it has been shown that
the classes of finite automata~\citep{Ang87}, monotone $\DNF$
formulas~\citep{Ang88}, polynomial-size decision trees~\citep{Bsh93}, and sparse
polynomials over GF(2)~\citep{SS96} are learnable in polynomial time. In a celebrated
result, \cite{Jackson:97} showed that the class of $\DNF$ formulas is
learnable in the $\PACplusMQ$ model under the uniform distribution.
\cite{Jackson:97} used Fourier analytic techniques to prove this result
building upon previous work of \cite{KushilevitzMansour:93} on learning
decision trees using membership queries under the uniform distribution. \medskip

\noindent {\bf Our Model}: Despite several interesting theoretical results, the
membership query model has not been received enthusiastically by machine
learning practitioners. Of course, there is the obvious difficulty of getting
labelers to perform their task while the learning algorithm is being executed.
But another, and probably more significant, reason for this disparity is that
quite often, the queries made by these algorithms are for labels of points that
do not look like typical points sampled from the underlying distribution. This
was observed by \cite{Lang92}, where experiments on handwritten
characters and digits revealed that the query points generated by the algorithms
often had no structure and looked meaningless to the human eye. This can cause
problems for the learning algorithm as it may receive noisy labels for such
query points.

Motivated by the above observations, we propose a model of membership queries
where the learning algorithm is restricted to query labels of points that
``look'' like points drawn from the distribution. In this paper, we focus our
attention to the case when the instance space is the boolean cube, i.e. $X =
\{-1, 1\}^n$, or $X = \zo^n$. However, similar models could be defined in the
case when $X$ is some subset of $\reals^n$. Suppose $x$ is a \emph{natural}
example, \ie one that was received as part of the training dataset (through the
oracle $\EX(f, D)$). We restrict the learning algorithm to make queries
$x^\prime$, where $x$ and $x^\prime$ are close in Hamming distance.  More
precisely, we say that a membership query $x^\prime$ is $r$-local with respect
to a point $x$, if the Hamming distance, $|x - x^\prime|_H$, is at most $r$.

One can imagine settings where these queries could be realistic, yet powerful.
Suppose you want to learn a hypothesis that predicts a particular medical
diagnosis using patient records. It could be helpful if the learning algorithm
could generate a new medical record and query its label. However, if the
learning algorithm is entirely unconstrained, it might come up with a record
that looks gibberish to any doctor. On the other hand, if the query chosen by
the learning algorithm is obtained by changing an existing record in a few
locations (local query), it is more likely that a doctor may be able to make
sense of such a record. In fact, this might be a powerful way for the learning
algorithm to identify the most important features of the record.

It is interesting to study what power these local membership queries add to the
learning setting. At the two extremes, are the $\PAC$ model (with $0$-local
queries), and $\MQ$-model (with $n$-local queries). It can be easily observed
that using only $1$-local queries, the class of parities can be learned in
polynomial time even in the presence random classification noise.  This 
problem is known to be notoriously difficult in the $\PAC$ learning
setting~\citep{BKW00}.  At the same time, most $\PACplusMQ$ algorithms we
are aware of, such as the algorithms for learning decision trees~\citep{Bsh93}
and for learning $\DNF$ formulas~\citep{Jackson:97}, rely crucially on using
$\MQ$s in
a strongly non-local way.  Also, it is easy to show that in a formal sense,
allowing a learner to make $1$-local queries gives it strictly more power than
in the $\PAC$ setting. In fact, essentially the same argument can be used to
show that $r+1$-local queries are more powerful than $r$-local queries.  These
separation results can be easily proved under standard cryptographic
assumptions, and are presented in Appendix~\ref{sec:separation}.
%% Uncomment FOR LONG VERSION
% We would like to contrast our model with the popular active learning
% framework~\cite{Seung92}.  An active learning algorithm can choose unlabeled
% points sampled from the distribution on which to query the oracle. Although
% they lead to reduced (labelled) sample complexity, such learning algorithms
% tend to be computationally inefficient.  Additionally, the active learning
% model is weaker than $\PAC$-learning, and hence, we cannot hope to actively
% and efficiently learn richer classes like decision trees without major
% breakthroughs in PAC learning results. \medskip
%
%

Our results are for learning on \emph{locally smooth} distributions over the
boolean cube, which we denote by $\moo^n$ (or sometimes by $\zo^n$). We say that
a distribution, $D$, over the boolean cube, $X = \{b_0, b_1\}^n$ is
\emph{locally $\alpha$-smooth} if for any two points $x$ and $x^\prime$ which
differ in only one bit, $D(x)/D(x^\prime) \leq \alpha$. Intuitively, this means
that points that are close to each other cannot have vastly different
frequencies. Frequency of a point reflects (and sometimes defines) its
``naturalness" and so, in a sense, our assumption on the distribution is the
same as the assumption underlying the use of local queries.  Note that the
uniform distribution is \emph{locally smooth} with $\alpha = 1$.  We will be
interested in the class of locally smooth distributions for a constant $\alpha$;
under these distributions changing $O(\log(n))$ bits can change the weight of a
point by at most a polynomial factor. Such distributions include product
distributions when the mean of each bit is some constant bounded away from $\pm
1$ (or $0, 1$). Convex combinations of locally $\alpha$-smooth distributions are
also locally $\alpha$-smooth. Alternatively, locally $\alpha$-smooth
distributions can be defined as the class of distributions for which the logarithm
of the density function is $\log(\alpha)$-Lipschitz with respect to the Hamming distance.

\eat{
%VITALY I think it quite clear and would perhaps make more sense in the
%"techniques" section.
Smooth distributions share some crucial properties with the uniform
distribution, which we exploit.  One such property is that the probability mass
when $d$ variables are fixed is in the range $[c_1^d, c_2^d]$ for some constants
$c_1 < c_2$. On the other hand, these distributions could be very far from the
uniform distribution.}

\medskip
\noindent {\bf Our Results:}
%\noindent
%the following: Start with an
%arbitrary (not necessarily smooth) distribution, $D$, and then add some
%independent random noise to each bit.
We give several learning algorithms for general locally smooth distributions and
for the special case of product/uniform distributions. Our main result for the
general locally smooth distributions is that sparse\footnote{Sparsity
refers to the number of non-zero coefficients.} polynomials are efficiently learnable with membership queries that are logarithmically local.

\begin{theorem} \label{thm:intro3} The class of $t$-sparse polynomials (with real
coefficients) over $\{0,1\}^n$ is efficiently learnable under the class of
locally $\alpha$-smooth distributions, for any constant $\alpha$, by a learning
algorithm that only uses $O(\log(n) + \log(t))$-local membership queries.
\end{theorem}

An important subclass of sparse polynomials is $O(\log n)$-depth decision trees.
For this subclass we also give a conceptually simpler analysis of our algorithm
(Appendix~\ref{sec:low-degree}).
Richer concept classes are also included in the class of sparse polynomials.
This includes the class of disjoint $\log(n)$-$\DNF$ expressions and $\log$-depth
decision trees, where each node is a monomial (rather than a variable). A
special case of such decision trees is $O(\log(n))$-term $\DNF$ expressions.

When the polynomials represent boolean functions this algorithm can easily be
made to work in the presence of \emph{persistent} random classification noise,
as described in Appendix~\ref{sec:rcn}.

For the special case of constant bounded product distributions we show that polynomial-size decision trees are efficiently learnable.
\begin{theorem} \label{thm:intro2} Let $\DP$ be the class of product
distributions over $X = \moo^n$, such that the mean of each bit is bounded away
from $-1$ and $1$ by a constant.  Then, the class of polynomial-size decision
trees is learnable with respect to the class of distributions $\DP$, by an
algorithm that uses only $O(\log(n))$-local membership queries.  \end{theorem}

We also consider polynomial size $\DNF$ which are known to be learnable in
$\PACplusMQ$.  \begin{theorem} \label{thm:intro4} The class of polynomial sized
$\DNF$ formulas is learnable under the uniform distribution using
$O(\log(n))$-local queries in time $n^{O(\log\log n)}$.
\end{theorem}

\medskip
\noindent {\bf Techniques}: All our results are based on learning polynomials.
It is well known that $\log$-depth decision trees can be expressed as sparse
polynomials of degree $O(\log(n))$.
%
% MAY BE REMOVE
\eat{
We point out that when considering $O(\log(n))$-degree polynomials, whether a
boolean variable is considered as taking a value in $\moo$ or $\zo$ does not
make a difference (up to polynomial factors).  However, if the constraint on
degree is removed, the class functions that can be defined as sparse polynomials
over $\zo^n$ is different from the class of functions that can be represented
as sparse polynomials over $\moo^n$.}

Our results on learning sparse polynomials (Section~\ref{sec:learn_poly}) under
locally smooth distributions rely on being able to identify all the important
monomials (those with non-zero coefficient) of low-degree, using
$O(\log(n))$-local queries. We identify a \emph{set} of monomials, the size of
which is bounded by a polynomial in the required parameters, which includes all
the important monomials.  The crucial idea is that using $O(\log(n))$-local
queries, we can identify given a subset of variables $S \subseteq [n]$, whether
the function on the remaining variables (those in $[n] \setminus S$), is zero or
not. We use the fact that the distribution is locally smooth to show that
performing $L_2$ regression over the set of monomials will give a good
hypothesis.

For uniform (or product) distributions (Section~\ref{sec:DNF_unif};
Appendices~\ref{sec:DT_unif}, \ref{sec:DT_prod}), we can make
use of Fourier techniques and numerous algorithms based on them.  A natural
approach to the problem is to try to adapt the famous algorithm of
\citet{KushilevitzMansour:93}  for learning decision trees (the KM algorithm) to
the use of local $\MQ$s. The KM algorithm relies on a procedure that isolates all
Fourier coefficients that share a common prefix and computes the sum of their
squares. Isolation of coefficients that share a prefix of length $k$ requires
$k$-local $\MQ$s and therefore we cannot use the KM algorithm directly. Instead we
isolate Fourier coefficients that contain a certain set of variables and we grow
these sets variable-by-variable as long as the sum of squares of coefficients in
them is large enough. Using $k$-local $\MQ$s it is possible to grow sets up to size
$k$. More importantly, the use of prefixes in the KM algorithm ensures that
Fourier coefficients are split into disjoint collections and therefore not too
many collections will be relevant. In our case the collections of coefficients
are not disjoint and so to prove that our algorithm will not take
superpolynomial time we rely on strong concentration properties of the Fourier
transform of decision trees.

For the case of $\DNF$ formulas, we use the result of ~\citet{F12, Kalai09} which
shows that one can learn a $\DNF$ formula given its heavy logarithmic-degree
Fourier coefficients. To recover those coefficients we use the same algorithm as
in the decision tree case. However in this case a more delicate analysis is
required to obtain even the $n^{O(\log\log n)}$ running time bound we give in
Theorem~\ref{thm:intro4}.  We rely on a concentration bound by
\citet{M92} that shows that the total weight of Fourier coefficients of
degree $d$ decays exponentially as $d$ grows.  We remark that Mansour also gives
a $\PACplusMQ$ algorithm for learning $\DNF$ running in time $n^{O(\log\log n)}$ but
aside from the use of the concentration bound our algorithm and analysis are
different (the dependence on the error $\eps$ in our algorithm is also
substantially better).

All known efficient algorithms for learning $\DNF$ under the uniform distribution
rely on agnostic learning of parities using the KM algorithm (or a related
algorithm of \citet{Levin:93}) \citep{BFJ+94,Jackson:97}. In the agnostic case
one cannot rely on the concentration properties crucial for our analysis and
therefore it is unclear whether poly-size $\DNF$ formulas can be learned
efficiently from logarithmically-local $\MQ$s. As some evidence of the hardness of
this problem, in Section \ref{sec:lower-bound} we show that for a constant $k$,
$k$-local queries do not help in agnostic learning under the uniform
distribution.

One point to note is that under locally $\alpha$-smooth distributions for a
constant $\alpha$, the main difficulty is designing algorithms which are faster
than time $n^{O(\log(n))}$. Designing an $n^{O(\log(n))}$ time algorithm is
trivial for decision trees and $\DNF$ formulas. In fact, one does not even require local-membership queries to do this. This follows from the observation that
\emph{agnostic} learning of $O(\log(n))$-size parities is easy in
$n^{O(\log n)}$ time. \medskip

\noindent {\bf Related work}: Models that address the problems that arise when
membership queries are answered by humans have been studied before. The work of
\citet{Blum98} proposed a noise
model wherein membership queries made on points lying in the low probability
region of the distribution are unreliable.  For this model the authors design
algorithms for learning an intersection of two halfspaces in $\reals^n$ and also
for learning a very special subclass of monotone $\DNF$ formulas.  Our result on
learning sparse polynomials can be compared with that of \citet{SS96}, who
provided an algorithm to learn sparse polynomials over GF(2) under arbitrary
distributions in Angluin's \emph{exact} learning model. However, their algorithm
is required to make membership queries that are not local.  \citet{Bsh93} gave an
algorithm for learning decision trees using membership queries. In both these
cases, it seems unlikely that the algorithms can be modified to use only local
membership queries, even for the class of locally smooth distributions.
\eat{% VITALY this can be removed since a more detailed discussion was added}
\citet{KushilevitzMansour:93} gave an algorithm for
learning decision trees under the uniform distribution using membership queries.
Their algorithm guarantees something stronger, viz.  agnostic learning of
parities.  While our decision tree learning algorithm for the uniform
distribution uses ideas from their work, we are unable to prove the stronger
result of agnostic learning (even $O(\log(n))$-sized) parities using local
membership queries.
}

There has been considerable work investigating learnability beyond the $\PAC$
framework. We consider our results in this body of work. Many of these models
are motivated by theoretical as well as real-world interest. On the one hand, it
is interesting to study the minimum extra power one needs to add to the $\PAC$
setting, to make the class of polynomial-size decision trees or $\DNF$ formulas
efficiently learnable. The work of~\citet{Aldous90} studies models of learning where the examples are generated according to a Markov process. An interesting special case of such models is when
examples are generated by a random walk on $\{-1,1\}^n$. For this model \citet{Bshouty05}
give an algorithm for learning $\DNF$ formulas (see also \citep{JacksonW09} for more recent developments). One could simulate random walks of length up to $O(\log(n))$ using
local membership queries, but adapting their $\DNF$ learning algorithm
to our model runs into the same issues as adapting the KM algorithm. The work of
\citet{Kalai09} provided polynomial time
algorithms for learning decision trees and $\DNF$ formulas in a framework where the
learner gets to see examples from a {\em smoothed} distribution.\footnote{The
notion of \emph{smoothness} in the work of Kalai et al. is not related to ours.
They consider product distributions where each bit has mean that is chosen randomly from a range bounded away from $\pm1$ by a constant.} Their model was inspired by the celebrated
smoothed analysis framework~\citet{Spielman04}. On the other hand, other models
have been proposed to capture plausible settings when the learner may indeed
have more power than in the $\PAC$-setting.  These situations arise for example
in scientific studies where the learner may have more than just \emph{black-box}
access to the function. Two recent examples in this line of work are the
learning using injection queries of \citet{AACW06}, and learning
using restriction access of \citet{DRWY12}.
\eat{ %VITALY: I think this sentence would make more sense in the conclusion section. But conclusion already makes similar points so this one is redundant
While our model is very
much a \emph{black-box} model, with the availability of crowdsourcing techniques
and increased potential of on-line labelers, the model we
consider may very well prove to be increasingly useful. }\medskip
%
%\eat{
%  TODO: I think there is a hardness result for learning
%low-degree sparse polynomials using low degree sparse polynomials under
%arbitrary distributions.
%}

%\textcolor{blue}{Need to change this based on the structure}
%
\noindent {\bf Organization}: Section \ref{sec:notation} introduces notation,
preliminaries and also formal definitions of the model we introduce in this
paper. Section~\ref{sec:learn_poly} presents the result on learning sparse
multi-linear polynomials. Due to space limitations, Appendix~\ref{sec:dts}
contains our results on learning decision trees, and also the implementation of
these algorithms in the presence of random classification noise.
Section~\ref{sec:DNF_unif} presents our algorithm for learning $\DNF$s under the
uniform distribution. Section~\ref{sec:lower-bound} contains the result showing
that agnostic learning in the $\PAC$ setting and with constant local queries is
equivalent.  Appendix~\ref{sec:separation} shows that the model we introduce is
strictly more powerful than the $\PAC$ setting, and strictly weaker than the
$\MQ$ setting. Finally, Section \ref{sec:conclusion} discusses directions for
future work.
%%%% COMMENT OUT FOR LONG VERSION
%Missing proofs are provided in the long version.

%% file: setting.tex
\noindent {\bf Notation}: Let $X$ be an instance space. In this paper, $X$ is
the boolean hypercube.  In Section~\ref{sec:DNF_unif} and
Appendix~\ref{sec:dts}, we will use $X = \{-1, 1\}^n$, as we apply Fourier
techniques. In Section~\ref{sec:learn_poly}, we will use $X = \{0, 1\}^n$ (the
class of sparse polynomials over $\{0, 1\}^n$ is different from sparse
polynomials over $\{-1, 1\}^n$). A concept class, $C$, is a set of functions
over $X \rightarrow Y$ (where $Y = \{-1, 1\}$ or $Y = \reals$). For a
distribution, $D$, over $X$ and any hypothesis, $h : X \rightarrow \{-1, 1\}$,
we define, $\err_D(h, f) = \Pr_{x \sim D}[h(x) \neq f(x)]$. If $h : X
\rightarrow \reals$, we use squared loss as the error measure, \ie 
%$\lerr_D(f, h) = 
$\E_{x \sim D}[(f(x) - h(x))^2]$. 
%
% UNCOMMENT FOR LONGER VERSION
% To simplify presentation, we will keep the parameter, $n$, representing the size
% of the instance space implicit, rather than considering families of instance
% spaces and concept classes defined for all values of $n$. 

For some bit vector $x$ (where bits may be $\zo$ or $\moo$), and any subset $S
\subseteq [n]$, $x_S$ denotes the bits of $x$ corresponding to the variables, $i
\in S$. The set $\bar{S}$ denotes the set $[n] \setminus S$. For two disjoint sets,
$S, T$, $x_Sx_T$ denote the variables corresponding to the set $S \cup T$. In
particular, $x_{S}x_{\bar{S}} = x$. 

If $D$ is a distribution over $X$, for a subset $S$, $D_S$ denotes the marginal
distribution over variables in the set $S$. Let $b_S$ denote a function, $b_S :
S \rightarrow \{b_0, b_1\}$, (where $\bzo = \zo$ or $\bzo = \moo$). Then, $x_S =
b_S$, denotes that for each $i \in S, x_i = b_S(i)$, thus the variables in the
set $S$ are set to the values defined by the function $b_S$. Let $\pi : X
\rightarrow \{0, 1\}$ denote some property (e.g. $\pi(x) = 1$, if $x_S = b_S$
and $\pi(x) = 0$ otherwise). The distribution $(D | \pi)$, denotes the
conditional distribution, given that $\pi(x) = 1$, \ie the property holds. \medskip

For the unfamiliar reader, Appendix~\ref{app:notation} provides standard
definitions of the $\PAC$~\citep{V84} model and the $\PACplusMQ$ model. \medskip \\

\noindent{\bf Local Membership Queries}:
For any point $x$, we say that a query $x^\prime$ is $r$-local with respect to
$x$ if the Hamming distance, $|x - x^\prime|_H$ is at most $r$. In our model, we
only allow algorithms to make queries that are $r$-local with respect to some
example that it received by querying $\EX(f, D)$, an oracle that returns a
random example from $D$ labeled according to $f$. We think of examples coming
through $\EX(f, D)$ as \emph{natural} examples. Thus, the learning algorithm
draws a set of natural examples from $\EX(f, D)$ and then makes queries that are
close to some example from this set. The queries are made to the membership
oracle, $\MQ(f)$, which on receiving a query $x$, returns $f(x)$. Formally, we
define learning using $r$-local membership queries as follows:

\begin{definition}[$\PAC$+$r$-local $\MQ$ Learning] Let $X$ be the instance
space, $C$ a concept class over $X$, and $\DC$ a class of distributions over
$X$. We say that $C$ is $\PAC$-learnable using $r$-local membership queries with
respect to distribution class, $\DC$, if there exist a learning algorithm,
$\LA$, such that for every $\epsilon > 0$, $\delta > 0$, for every distribution
$D \in \DC$ and every target concept $f \in C$, the following hold: 
\begin{enumerate}
\item $\LA$ draws a sample, $\mc{S}$, of size $m = \mathrm{poly}(n, 1/\delta,
1/\epsilon)$ using example oracle, $\EX(f, D)$
\item Each query, $x^\prime$, made by $\LA$ to the membership query oracle,
$\MQ(f)$, is $r$-local with respect to some example, $x \in \mc{S}$
\item $\LA$ outputs a hypothesis, $h$, that satisfies with probability at least
$1 - \delta$, $\err_D(h, f) \leq \epsilon$
\item The running time of $\LA$ (hence also the number of oracle accesses) is
polynomial in $n$, $1/\epsilon$, $1/\delta$ and the output hypothesis, $h$, is
polynomially evaluable.
\end{enumerate}
\end{definition} \medskip

\noindent{\bf Locally Smooth Distributions}: Since we want to talk about locally
smooth distributions over $\moo^n$ and $\zo^n$ both, we consider $X = \{b_0, b_1
\}^n$ and state the properties of interest in general terms.  We say that a
distribution, $D$, over $X = \{b_0, b_1\}^n$ is locally $\alpha$-smooth, for
$\alpha \geq 1$, if for every pair $x, x^\prime \in X$, with Hamming distance,
$|x - x^\prime|_H = 1$, it holds that $D(x)/D(x^\prime) \leq \alpha$. 

We will repeatedly use the following useful properties of locally $\alpha$-smooth
distributions. The proof of these are easy and hence are omitted.
\begin{fact} \label{fact:smooth} Let $D$ be a locally $\alpha$-smooth distribution over
$X = \bzo^n$. Then the following are true:
\begin{enumerate}
\item For $b \in \bzo$, $\frac{1}{1+ \alpha} \leq \Pr_{D}[x_i = b] \leq
\frac{\alpha}{1 + \alpha}$.
\item For any subset, $S \subseteq [n]$, the marginal distribution,
$D_{\bar{S}}$ is locally $\alpha$-smooth.
\item For any subset $S \subseteq [n]$, and for any property, $\pi_S$, that
depends only on variables $x_S$ (e.g. $x_S = b_S$), the marginal (with respect
of $\bar{S}$) of the conditional distribution, $(D|\pi_S)_{\bar{S}}$ is
locally $\alpha$-smooth.
\item (As a corollary of the above three) $\left( \frac{1}{1 + \alpha}
\right)^{|S|} \leq \Pr_D[x_S = b_S] \leq \left( \frac{\alpha}{1 + \alpha}
\right)^{|S|}$. 
%
%\item (As a corollary of the above four) For any $x \in \{ b_0, b_1 \}^n$, $D(x)
%\leq \left(\frac{1+\alpha}{2+\alpha}\right)^n$.
%
\end{enumerate}
\end{fact}
We use Fourier analytic techniques and a brief introduction is provided in
Appendix~\ref{app:notation}.

%% file: learn_poly_reals.tex
In this section, we consider the problem of learning $t$-sparse polynomials with
coefficients over $\reals$ (or $\rationals$), when the domain is restricted to
$\{0, 1\}^n$. In this case, we may as well assume that the polynomials are
multi-linear. We assume that the absolute values of the coefficients are bounded
by $B$, and hence the polynomials take values in $[-tB, tB]$, on the domain
$\{0, 1\}^n$. For a subset $S \subseteq [n]$, let $\xi_S(x) = \prod_{i \in
S}x_i$, thus $\xi_{S}(x)$ is the monomial corresponding to the variables in
the set $S$. Note that any $t$-sparse multi-linear polynomial can be represented
as,
\[ f(x) = \sum_{S} c_S \xi_S(x), \]
where $c_S \in \reals$,  $|\{S ~|~ c_S \neq 0 \}| \leq t$, and $|c_S| \leq B$
for all $S$. Let $\reals^n_{t, B}[X]$ denote the class of multi-linear
polynomials over $n$ variables with coefficients in $\reals$, where at most $t$
coefficients are non-zero and all coefficients have magnitude at most $B$.

We assume that we have an infinite precision computation model for
reals.\footnote{The case when we have bounded precision can be handled easily
since our algorithms run in time polynomial in $B$, but is more cumbersome.}
Also, since the polynomials may take on arbitrary real values, we use squared
loss as the notion of error. For a distribution, $D$ over $\{0, 1\}^n$, the
squared loss between polynomials, $f$ and $h$, is $\E_{x \sim D}[(f(x) -
h(x))^2]$. Our main result is:

\begin{theorem} \label{thm:learn_poly} The class $\reals^n_{t, B}[X]$, is
learnable with respect the class of $\alpha$-smooth distributions over $\{0,
1\}^n$, using $O(\log(n/\epsilon) + \log(t/\epsilon))$-local $\MQ$s and in time
$poly((ntB/\epsilon)^\alpha , \log(n/\delta))$. The output hypothesis is a
multi-linear polynomial, $h$, such that, with probability $(1-\delta)$, $\E_{x \sim D}[(h(x) - f(x))^2] \leq
\epsilon$.  \end{theorem}

Recall that for a subset, $S$, $x_S$ denotes the variables that are in $S$; and
that $\bar{S}$ denotes the set $[n] \setminus S$.  Let $f_S(x_{\bar{S}})$ denote
the multi-linear polynomial defined only on variables in $x_{\bar{S}}$, 
\[ f_S(x_{\bar{S}}) = \sum_{T \subseteq \bar{S}} c_{S \cup T} \xi_{T}(x_{\bar{S}}) \]

Here, we describe the high-level idea of the proof of
Theorem~\ref{thm:learn_poly}. The details of the argument are provided in
Appendix~\ref{app:learn_poly}. Algorithm~\ref{alg:learn_poly} outputs a
hypothesis that approximates the polynomial $f$. \medskip
\begin{figure}[!t]
\begin{center}
\fbox{
\begin{minipage}{0.95 \columnwidth}
{\bf Algorithm}: \textsc{Learning $t$-Sparse Polynomials} \smallskip \\
{\bf inputs}: $d$, $\theta$, oracles $\EX(f, D)$, (local)$\MQ(f)$ 
%\# $f$ is a polynomial in $\FF_t(X)$ \\
%\# $D$ is a $\alpha$-smooth distribution over $\zo^n$ \\
%\# $m$ is the size of sample to draw from $\EX(f, D)$\\
%\# $d$ is the degree bound for identifying monomials\\
%
\begin{enumerate}
\item {\bf let} $\mc{S} = \emptyset$
%\item {\bf let} $\langle (x^i, f(x^i)) \rangle_{i=1}^m$ denote a sample drawn
%from $\EX(f, D)$
\item {\bf repeat} (while some new set is added to $\mc{S}$)
\begin{enumerate}
\item  For every $S^\prime \in \mc{S}$, $|S^\prime| \leq d-1$ and for every
$j \in [n] \setminus S^\prime$
%%%
\begin{enumerate}
\item {\bf let} $S = S^\prime \cup \{j\}$
\item {\bf if} $\Pr_{D_{\bar{S}}}[f_S(x_{\bar{S}}) \neq 0] \geq \theta$, {\bf then}
$\mc{S} = \mc{S} \cup \{ S \}$
\end{enumerate}
%%%
\end{enumerate}
\item Perform regression to identify a polynomial $h = \sum_{S} h[S] \xi_S(x)$,
that minimizes $\E[(f(x) - h(x))^2]$, subject to:
\begin{enumerate}
\item $h[S] = 0$ for $S \not\in \mc{S}$.
\item $\sum_{S} |h[S]| \leq tB$
\end{enumerate}
\end{enumerate}
{\bf output} $h(x)$
\end{minipage}
}
\end{center}
\caption{Algorithm: Learning $t$-Sparse Polynomials \label{alg:learn_poly}}
\end{figure}

\noindent {\bf Truncation}: First, we show that there are low-degree polynomials
that approximate the multi-linear polynomial, $f$, up to arbitrary (inverse
polynomial) accuracy. These polynomials are the truncations of $f$ itself. Let
$f^d$ denote the multi-linear polynomial obtained from $f$ by discarding all the
terms of degree at least $d+1$. Note that $f^d$ is multi-linear and $t$-sparse,
and has coefficients of magnitude at most $B$. Thus, 
\begin{align*}
f^d(x) = \sum_{\substack{S \subseteq [n] \\ |S| \leq d}} c_S \xi_{S}(x) 
\end{align*}

Now, observe that because $D$ is locally $\alpha$-smooth, the probability that
$\xi_{S}(x) = 1$ is at most $(\alpha/(1 + \alpha))^{|S|}$ (see
Fact~\ref{fact:smooth}). Thus, the probability that at least one term of degree
$\geq d+1$ in $f$ is non-zero, is at  most $t(\alpha/(1 + \alpha))^d$ by a union
bound. Thus, $\Pr_{x \sim D}[f(x) \neq f^d(x)] \leq t(\alpha/(1 + \alpha))^d$.
Also, since $|f(x)| \leq tB$ and $|f^d(x)| \leq tB$, this implies that $\E_{x
\sim D}[(f(x) - f^d(x))^2] \leq 4t^3 B^2 (\alpha/(1+\alpha))^d$. By choosing $d$
appropriately, when $\alpha$ is a constant this quantity can be made arbitrarily
(inverse polynomial) small.

Step 3 of the Algorithm (see Fig.~\ref{alg:learn_poly}) identifies all the
important coefficients of the polynomial, $f$. Suppose, we could guarantee that
the set, $\mc{S}$, contains all coefficients $S$, such that $c_{S} \neq 0$ and
$|S| \leq d$, \ie all non-zero coefficients of $f^d$ are identified. This
guarantees that the regression in step 4 will give a good approximation to $f$,
since the error of the hypothesis obtained by regression has to be smaller than
$\E_{x \sim D}[(f(x) - f^d(x))^2]$. 
%USE THIS LINE IN THE LONG VERSION
%(The generalization guarantees are fairly standard and are described later.) \smallskip
%USE THIS LINE IN THE STOC VERSION
(The generalization guarantees are fairly standard and are described in the long
version.) \smallskip

\noindent{\bf Identifying Important Monomials}: In order to test whether or not
a monomial, $S$, is important, the algorithm checks whether
$\Pr_{D_{\bar{S}}}[f_{S}(x_{\bar{S}}) \neq 0] \geq \theta$. Here, $D_{\bar{S}}$ is the marginal
distribution over the variables $x_{\bar{S}}$. We assume that this test can be
performed perfectly accurately. (The analysis using samples is standard by
applying appropriate Chernoff-Hoeffding bounds.)

In Lemma~\ref{lem:all-important}, we show that if the polynomial $f_S(x_{\bar{S}})$
has a non-zero coefficient of degree at most $d - |S|$, then the probability
that $f_S(x_{\bar{S}}) \neq 0$ is at least $(1/(1+\alpha))^{d + \log(t)}$. In
Lemma~\ref{lem:not-too-many}, we show that if $f_S(x_{\bar{S}})$ has no non-zero
coefficient of degree less than $d^\prime = O((d + \ln(t))\ln(1+\alpha))$, then
$f_S(x_{\bar{S}}) \neq 0$ with probability at most $0.5 (1/(1+ \alpha))^{d +
\log(t)}$. Thus, we will never add any subset $S$, unless there is some
co-efficient $T$ in $f$ of size at most $d^\prime = O((d + \ln(t))\ln(1+
\alpha))$ and $S \subseteq T$. However, the number of such $T$ is at most $t$,
and each such set can have at most $2^{d^\prime}$ subsets. This bounds the total
number of subsets the algorithm may add to $\mc{S}$, and hence, also the running
time of the algorithm (to polynomial in the required parameters).

Note that sampling, $x_{\bar{S}}$ according to $D_{\bar{S}}$ is trivial, just draw random
example from $\EX(f, D)$ and ignore the variables $x_S$. Let $U_{S}$ denote the
uniform distribution over variables in $x_S$. Then we have,
\begin{align}
f_{S}(x_{\bar{S}}) = \E_{x_S \sim U_S}[2^{|S|} \prod_{i \in S}(2x_i - 1)f(x)]
\label{eqn:computing_f_S}
\end{align}
The variables in $x_{\bar{S}}$ are fixed, the expectation is only taken over the
uniform distribution over variables in $x_S$. Notice that for any $i \in S$,
since $x_i \in \{0, 1\}$, $\E_{x_S}[(2x_i - 1)] = 0$ and $\E_{x_S}[(2x_i -1)x_i]
= 1/2$. Thus, in the RHS of (\ref{eqn:computing_f_S}), if $S \not\subseteq T$,
$\E_{x_S}[2^{|S|} \prod_{i \in S}(2x_i -1) \xi_T(x)] = 0$, and if $S \subseteq
T$, $\E_{x_S}[2^{|S|} \prod_{i \in S}(2x_i -1)\xi_T(x)] = \xi_{T \setminus
S}(x_{\bar{S}})$.  Thus, the relation in (\ref{eqn:computing_f_S}) is true. Also,
this means that if the example, $x$, is received by querying the oracle, $\EX(f,
D)$, $f_{S}(x_{\bar{S}})$ can be obtained by making $O(|S|)$-local membership queries
to the oracle $\MQ(f)$.

The complete details of the proof appear in Appendix~\ref{app:learn_poly}.

%% USE THIS LINE FOR LONG VERSION
%We now give formal proofs of the ideas explained above.
% USE THIS LINE FOR STOC VERSION
%We now describe some more details of the proof of the main result. For the
%complete proof, please refer to the long version.

%% COMMENT OUT THIS SECTION FOR STOC VERSION
% \subsection{Truncation}
% 
% We show that truncation (to $\log$-degree) does not change the polynomial
% significantly, under smooth distributions.
% 
% \begin{lemma} \label{lem:truncate} Let $f$ be a $t$-sparse multi-linear
% polynomial and let $D$ be an $\alpha$-smooth distribution. Let $f^d$ be the
% polynomial $f$ where all terms of degree greater than $d$ are set to $0$. Then,
% %
% \[ \Pr_D[f_d(x) \neq f(x)] \leq t \left(\frac{\alpha}{1 + \alpha} \right)^d, \]
% %
% and hence,
% \[ \E_D[(f(x) - f^d(x))^2] \leq 4 t^3 B^2 \left(\frac{\alpha}{1 +
% \alpha}\right)^d \]
% \end{lemma}
% \begin{proof} Note that for any $x$, if $f(x) \neq f_d(x)$, there must be a term
% of $f$ of degree at least $d$ that is not $0$. For any fixed monomial of $f$ of
% degree at least $d$, the probability that it is non-zero for a random point of
% $x$ drawn under an $\alpha$-smooth $D$ is at most $(\alpha/(1 + \alpha))^d$ (see
% Fact~\ref{fact:smooth}). Taking a union bound over the $t$ possible terms gives
% the result.  \end{proof}

%% file: DNF_unif.tex
% LONG VERSION
% In this section, we present an algorithm for learning $t$-leaf decision trees
% (of arbitrary depth) under the uniform distribution. Although, the uniform
% distribution is a special case of product distributions considered in Section
% \ref{sec:DT_prod}, the exposition is simpler and conveys the high-level ideas
% better.
%
%% STOC VERSION
In this section, we present an algorithm for learning polynomial size $\DNF$
formulas under the uniform distribution.
%We outline how this can be
%extended to product distributions. The full details are provided in the long
%version of this paper.
%
%%
We will frequently use the following facts about $\DNF$ formulas.

\begin{enumerate}
\item For every size-$s$ $\DNF$ formula $f$, there exists a size-$s$, $\DNF$
formula $g$ with terms of size $l$, such that $\Vert f - g \Vert_2^2
= \E[(f(x) - g(x))^2] \le \frac {4s} {2^l}$. This follows by setting
$g$ to the $\DNF$ formula obtained by dropping all terms of size greater
than $l$ from $f$.
\item Let $f$ be an $l$-term $\DNF$ formula. Then $\sum_{|S| > t} \hat{f}^2(S) \le
2^{\frac {-t} {10l}}$. This fact follows from the proof of Lemma 3.2 in
\citep{M92}.
\end{enumerate}

We will further use the following result from \citep{Kalai09, F12}:
\begin{theorem}[\citep{F12, Kalai09}] \label{thm:DNF-learn} If $f$ is a size-$s$
$\DNF$ formula, then there exists an efficient randomized algorithm,
$\mathsf{LearnDNF}$ that given access to the heavy, low-degree Fourier
coefficients, \ie the set $\{ \hat{f}(S) ~|~ |S| \leq \log(s/\epsilon),
|\hat{f}(S)| \geq (\epsilon/(4s)) \}$, outputs a hypothesis, $h$, such that
$\err_{U}(h, f) \leq \epsilon$.  \end{theorem}
%
%
% COMMENT OUT FOR STOC VERSION
% \begin{proof}
% Consider,
% %
% \begin{align*}
% \sum_{S, |S| \geq \log(t^2/\tau)} \hat{f}(S)^2 &\leq \max_{S, |S| \geq
% \log(t^2/\tau)} |\hat{f}(S)| \cdot \left( \sum_{T, |T| \geq \log(t^2/\tau)}
% |\hat{f}(S)| \right) \\
% %
% &\leq t \cdot (\tau / t^2) \cdot L_1(f) \leq \tau
% \end{align*}
% \end{proof}

The algorithm to obtain all \emph{heavy, low-degree terms} is  presented in
Figure \ref{alg:unif-dnf}. The output of the algorithm,
$\mathsf{LearnDNF}({\mathcal S})$ is obtained by doing the following: (i)
estimate all the Fourier coefficients, $\hat{f}(S)$, for $S \in {\mathcal S}$
(ii) use the algorithm from Theorem~\ref{thm:DNF-learn} to obtain $h$. The rest
of the section is devoted to the proof of the following Theorem.

\begin{theorem} \label{thm:DNF-local} The class of size-$s$ $\DNF$ formulas is
learnable in time $(s/\epsilon)^{O(\log\log(s/\epsilon))}$ using
$O(\log(s/\epsilon))$-local membership queries under the uniform distribution.
\end{theorem}

\begin{figure}[!t]
\begin{center}
\fbox{
\begin{minipage}{0.95 \columnwidth}
{\bf Algorithm}: \textsc{Learning $\DNF$ formulas} \smallskip \\
~~{\bf inputs}: $d$, $\theta$, oracles $\EX(f, U)$, (local)-$\MQ(f)$
%~~\texttt{\#} $f$ is a size-$s$ decision tree \\
%~~\texttt{\#} $U$ is the uniform distribution over $\moo^n$.
%
\begin{enumerate}
\item {\bf let} $\mc{S} = \{ \emptyset \}$
\item {\bf for} $i = 1, \ldots, d$
\begin{enumerate}
\item {\bf for} every $S^\prime \in \mc{S}$, $|S^\prime| = i-1$ and for every $j \in
[n] \setminus S^\prime$
\begin{enumerate}
\item {\bf let} $S = S^\prime \cup \{j \}$
\item ($L_2$ Test){ \bf if} $\E_{x \sim U}[f_S(x)^2] > \theta^2$, {\bf then} $\mc{S} = \mc{S} \cup \{ S
\}$
\end{enumerate}
\end{enumerate}
\end{enumerate}
%\item {\bf let} $h(x) = \sum_{S \in \mc{S}} \hat{f}(S) \chi_S(x)$ \end{enumerate}
~~{\bf output}: $h = \mathsf{LearnDNF}({\mathcal S})$
\end{minipage}
}
\caption{Algorithm: Learning $\DNF$ formulas under the Uniform
Distribution\label{alg:unif-dnf}}
\end{center}
\end{figure}
%
%
%% UNCOMMENT FOR LONG VERSION
%The rest of the section is devoted to a formal proof of the above overview.
%
%% UNCOMMENT FOR STOC VERSION
When $d = log(s/\epsilon)$ and $\theta = \epsilon/4s$, we argue that every
Fourier coefficient that has magnitude at least $\theta$ is included in
${\mathcal S}$ and also that $|{\mathcal S}|$ is
$(s/\epsilon)^{O(\log\log(s/\epsilon))}$. The first part is immediate (see
Claim~\ref{claim:T-in-S} in App.~\ref{sec:DT_unif}) and the second is proved in
Claim~\ref{claim:not-too-many-DNF}. Note that the proof of
Theorem~\ref{thm:DNF-local} follows immediately using Claims~\ref{claim:T-in-S},
\ref{claim:not-too-many-DNF} and Theorem~\ref{thm:DNF-learn}. The proof of
Claim~\ref{claim:not-too-many-DNF} is provided in Appendix~\ref{app:DNF_unif}.
%To plug into Vitaly's argument we need to argue that every term of coefficient value more than $\theta$ and degree $\le d$ will be added. This is clear from the lemma below which is the same as before.

%\begin{claim}
%\label{claim:T-in-S-DNF}
%Suppose that $S$ is such that $|\hat{f}(S)| \geq \theta$ and $|S| \leq d$, then
%$S \in \mc{S}$.
%\end{claim}
% COMMENT OUT PROOF FOR STOC VERSION
% \begin{proof}
% First observe that for any subset $S^\prime \subseteq S$, it holds that
% $\E[f_{S^\prime}(x)^2] \geq \theta^2$. This follows immediately by observing
% that
% %
% \begin{align*}
% \E[f_{S^\prime}(x)^2] &= \sum_{T \supseteq S^\prime} \hat{f}(T)^2 \geq
% \hat{f}(S)^2 \geq \theta^2
% \end{align*}
% %
% It follows by a simple induction argument that at iteration $i$, $\mc{S}$
% contains every subset of $S$ of size at most $i$, for which $\E[f_{S}(x)^2] \geq
% \theta^2$. And, hence $S \in \mc{S}$.
% \end{proof}
%
%Next we need to argue that we will not add too many terms.
\begin{claim} \label{claim:not-too-many-DNF}
After running algorithm~\ref{alg:unif-dnf}, $|{\mathcal S}| = (s/\epsilon)^{O(\log\log(s/\epsilon))}.$
\end{claim}
%\begin{proof}
%See Section~\ref{app:DNF_unif}
%\end{proof}
%

%% COMMENT OUT PROOF FOR STOC VERSION
% \begin{proof} Since $S \in \mc{S}$, we know that $\E[f_S(x)^2] = \sum_{S
% \supseteq S} \hat{f}(S)^2 \geq \theta^2$. But observe that,
% %
% \begin{align*}
% \sum_{T \supseteq S} \hat{f}(T)^2 \leq \left(\sum_{T \supseteq S}
% |\hat{f}(T)| \right) \cdot \max_{T \supseteq S} |\hat{f}(T)|
% \end{align*}
% %
% The above inequality simply states the fact that  $L_2(f_S) \leq L_1(f_S)
% L_\infty(f_S)$. Since $f$ is a $t$-leaf decision tree, $\sum_{T \supseteq
% S}|\hat{f}(T)| \leq L_1(f) \leq t$. The claim now follows immediately.
% \end{proof}
%

%% file: lower_bound.tex
In this section, we prove that any concept class, $C$, efficiently
\emph{agnostically} learnable over the uniform distribution with (constant)
$k$-local $\MQ$s  is also efficiently \emph{agnostically} learnable from random
examples alone. This result can be compared with that of \citet{F08}, where it
is shown that membership queries do not help for distribution-independent
agnostic learning. %However, it was shown there that under uniform distribution,
%membership queries do help under standard cryptographic assumptions.

We remark that $1$-local $\MQ$s suffice for learning parities of any size with random classification noise. At the same time when learning from random examples alone agnostic learning of parities can be reduced to learning parities with random classification noise~\citep{FGKP06}.
However this reduction does not lead to an agnostic algorithm for learning
parities with $1$-local $\MQ$s since it is highly-nonlocal: the (noisy) label of every example is influenced by labels of points chosen randomly and uniformly from the whole hypercube.

Our reduction is based on embedding the unknown function, $f$, in a higher
dimensional domain such that the original points are mapped to points that are
at least at distance $2k+1$ apart (and in particular that no single point in the
domain is $k$-close to more than one of the original points). A crucial property
of this embedding is that, up to scaling it preserves the correlation of any
function with $f$.  The embedding is achieved using a linear binary
error-correcting code, specifically we use the classic binary BCH
code~\citep{H59, BRC60}. The proof of the following theorem is provided in
Appendix~\ref{sec:lower-bound-app}.

\begin{theorem}
\label{thm:lower_bound}
 For any constant, $k$, if a concept class $C$ is learnable
agnostically under the uniform distribution in the $\PAC$+$k$-local $\MQ$ model,
then $C$ is also agnostically learnable in the $\PAC$ model.  \end{theorem}
%\begin{proof}
%See Section~\ref{sec:lower-bound-app}.
%\end{proof}
%
In particular, this implies that it is highly unlikely that the class of
parities (even of size $O(\log(n))$) will be efficiently agnostically learnable
using $k$-local $\MQ$s. The class of parities is of particular interest, because an
efficient agnostic algorithm for learning $O(\log(n))$ sized parities would
yield an efficient $\DNF$-learning algorithm.

%% file: conclusion.tex
We introduced the local membership query model, with the goal of studying query
algorithms that may be useful in practice. With the rise of crowdsourcing tools,
it is increasingly possible to get human labelers for a variety of tasks. Thus,
membership queries beyond the standard active learning paradigm could prove to
be useful to increase the efficiency and accuracy of learning. In order to make
use of human labelers, it is necessary to make queries that make sense to them.
In some ways, our algorithms can be understood as searching for
higher-dimensional (deeper) features using queries that modify the examples
locally.

Our model of local membership queries is also a very natural and simple
theoretical model. There are several interesting open questions: (i) can the
class of $t$-leaf decision trees (without depth restriction) be learned under the
class of locally smooth distributions? (ii) is the class of $\DNF$ formulas
learnable in polynomial time, at
least under the uniform distribution? Another interesting question is whether a
general purpose boosting algorithm exists that only uses locally $\alpha$-smooth
distributions. This looks difficult since most boosting algorithms decrease
weights of points substantially\footnote{Note that Smooth-boosting~\citep{Ser03}
does not use distributions that are smooth according to our definition}.

It is also interesting to see whether agnostic learning of any interesting
concept classes is possible in this learning model. Our results show that
constant local queries are not useful for agnostic learning. However can
$O(\log(n))$-local queries help in learning $O(\log(n))$-sized parities in the
agnostic setting? We observe that learning the class of $O(\log(n))$-sized
parities and the class of decision-trees is equivalent in the agnostic learning
setting (even under locally smooth distributions), since weak and strong agnostic
learning is equivalent even with respect to a fixed distribution \citep{KK09,
Feldman10}. Agnostic learning $O(\log(n))$-sized parities (even with respect to
a fixed distribution) would also imply ($\PAC$) learning $\DNF$ in our model with
local membership queries (with respect to the same distribution) \citep{KKM09}.

%% file: setting_app.tex
\noindent{\bf $\PAC$ Learning}~\citep{V84}: Let $\DC$ be a class of distributions
over $X$ and $D \in \DC$ be some distribution. Let $C$ be a concept class over
$X$, and $f \in C$. An example oracle, $\EX(f, D)$, when queried, returns $(x,
f(x))$, where $x$ is drawn randomly from distribution $D$. The learning
algorithm in the $\PAC$ model has access to an example oracle, $\EX(f, D)$,
where $f \in C$ is the unknown target concept and $D \in \DC$ is the target
distribution. The goal of the learning algorithm is to output a hypothesis, $h$,
that has low error with respect to the target concept under the target
distribution, \ie  $\err_D(h, f) = \Pr_{x \sim D} [ h(x) \neq f(x)] \leq
\epsilon$. \medskip 

\noindent {\bf Membership Queries}: Let $f \in C$ be a concept defined over
instance space $X$. Then a \emph{membership query} is a point $x \in X$. A
membership query oracle $\MQ(f)$, on receiving query $x \in X$, responds with
value $f(x)$. In the $\PACplusMQ$ model of learning, along with the example
oracle $\EX(f, D)$, the learning algorithm also has access to a membership
oracle, $\MQ(f)$. \medskip 

\noindent{\bf Fourier Analysis}: Here, we assume that $X = \{-1, 1\}^n$ (and not
$\zo^n$). For $S \subseteq [n]$, let $\chi_S : X \rightarrow \{-1, 1\}$ denote
the parity function on bits in $S$, \ie $\chi_S(x) = \prod_{i \in S} x_i$.  When
working with the uniform distribution, $U_n$, over $\moo^n$, it is well known
that $\langle \chi_S \rangle_{S \subseteq [n]}$ forms an orthonormal basis
(Fourier basis) for functions $f: X \rightarrow \reals$.  Hence $f$ can be
represented as a degree $n$, multi-linear polynomial over the variables $\{x_i\}$, 
\begin{align*}
f(x) &= \sum_{S \subseteq [n]} \hat{f}(S) \chi_S(x) 
\end{align*}
where $\hat{f}(S) = \E_{x \sim U_n}[ \chi_S(x) f(x) ]$.  Define $L_1(f) =
\sum_{S \subseteq [n]} |\hat{f}(S)|$, $L_2(f) = \sum_{S \subseteq [n]}
\hat{f}(S)^2$, $L_\infty(f) = \max_{S \subseteq [n]}|\hat{f}(S)|$ and $L_0(f) =
|\{ S \subseteq [n] ~|~ \hat{f}(S) \neq 0 \}|$. Parseval's identity, states that
$\E_{x \sim U_n}[f^2(x)] = L_2(f)$. For Boolean functions, \ie with range $\{-1,
1\}$, Parseval's identity implies that $\sum_{S \subseteq [n]} \hat{f}(S)^2 =
1$. Other useful observations that we use frequently are:  (i) $L_2(f) \leq
L_1(f) \cdot L_\infty(f)$; (ii) $L_2(f) \leq L_0(f) \cdot (L_\infty(f))^2$. \medskip

\noindent{\bf Polynomials}: In this paper, we are only concerned with
multi-linear polynomials, since the domain is $\moo^n$ or $\zo^n$. A
multi-linear polynomial over $n$ variables can be expressed as: 
\[ f(x) = \sum_{S \subseteq [n]} c_S \prod_{i \in S} x_i \] 
When the domain is $\moo^n$, the monomials correspond to the parity function,
$\chi_S(x)$, and if the distribution is uniform over $\moo^n$, $c_S = \E_{x \sim
U}[f(x)\chi_S(x)] = \hat{f}(S)$. When the domain is $\zo^n$, the monomials
correspond to conjunctions over the variables included in the monomial. We
denote such conjunctions by, $\xi_S(x) = \prod_{i \in S} x_i$. 

For any $S \subseteq [n]$, define $f_S(x) := \sum_{T \supseteq S} c_T \prod_{i
\in T \setminus S} x_i$, and $f_{-{S}}(x) = f(x)
- (\prod_{i \in S} x_i) \cdot f_S(x)$.

%% file: learn_poly_app.tex
%\section{Proof of Theorem~\ref{}}

\noindent \textbf{Proof of Theorem~\ref{thm:learn_poly}}\\
First, we have the following useful general lemma.

\begin{lemma} \label{lem:non-zero-prob} Let $f$ be a $t$-sparse multi-linear
polynomial defined over any field, $\FF$, with a non-zero constant term, $c_0$.
Let $D$ be any locally $\alpha$-smooth distribution over $\{0, 1\}^n$, then \[
\Pr_D[f(x) \neq 0] \geq \left(\frac{1}{1 + \alpha}\right)^{\log_2(t)} \]
\end{lemma}
%% COMMENT OUT PROOF FOR STOC VERSION
 \begin{proof}
 We prove this by induction on the number of variables, $n$. When $n = 1$, the
 only possible polynomials are $f(x_1) = c_0 + c_1 x_1$. Then $f(x) = 0$ if and
 only if $x_1 = 1$ and $c_1 = -c_0$ (since $c_0 \neq 0$). Note that when $D$ is
locally $\alpha$-smooth, $\Pr[x_1 = 1] \leq \alpha/(1 + \alpha)$ (see
 Fact~\ref{fact:smooth}). Thus, $\Pr_D[f(x_1) \neq 0] \geq {1}/{(1 + \alpha)}$.
 (And the sparsity is $2$, and $\log(2) = 1$.) Thus the base case is verified.
 
 Let $f$ be any multi-linear polynomial defined over $n$ variables. Suppose there
 exists a variable, without loss of generality, say $x_1$, such that $c_1 x_1$ is
 a term in $f$, where $c_1 \neq 0$. Then we can write $f$ as follows:
 \[ f(x) = f_{-1}(x) + x_1 f_1(x) \]
 where $f_{-1}$ and $f_1$ are both multi-linear polynomials over $n-1$ variables
 and both have a non-zero constant term. (The constant term of $f_{-1}$ is just
 $c_0$, and $f_1$ has constant term $c_1$.) Then note that $1/(1+\alpha)\leq
 \Pr_D[x_1= b | x_{-1}] \leq \alpha/(1 + \alpha)$, for both $b = 1$ and $b = 0$.
 Now, it is easy to see that $\Pr_D[f(x) \neq 0] \geq \Pr_D[x_1  = 0 |
 x_{-1}]\Pr_D[f_{-1}(x) \neq 0] \geq (1/(1+\alpha))\Pr[f_{-1}(x) \neq 0]$. 
 
 To see that $\Pr_D[f(x) \neq 0] \geq (1/(1+\alpha))\Pr_D[f_1(x) \neq 0]$
 consider the following: Fix $x_{-1}$, if $\Pr[f_1(x) \neq 0]$, then for at least
 one setting of $x_1$, it must be the case that $f(x) \neq 0$. Thus, conditioned
 on $x_{-1}$, $\Pr_D[f(x) \neq 0 | x_{-1}] \geq (1/(1+\alpha)) \delta(f_{-1}(x)
 \neq 0)$ (here $\delta(\cdot)$ is the indicator function). Thus, $\Pr_D[f(x)
 \neq 0] \geq (1/(1+\alpha))\Pr_D[f_{-1}(x) \neq 0]$.  However, at least one of
 $f_{-1}$, $f_1$ must have sparsity at most $t/2$, thus by induction we are done.
 
 In the case, that there is no $x_i$ such that $c_i x_i$ (with $c_i \neq 0$)
 appears in $f$ as a term, let $f_0$ be the polynomial obtained from $f$ by
 setting $x_1 = 0$ and $f_1$ be the polynomial obtained from $f$ by setting $x_1
 = 1$. Note that both $f_0$ and $f_1$ have constant term $c_0 \neq 0$ and
 sparsity at most $t$, but they have one fewer variable than $f$. Thus,
 $\Pr_D[f_b(x) \neq 0] \geq (1/(1+\alpha))^{\log_2(t)}$, for $b = 0, 1$.
 However, note that
 \begin{eqnarray*}
  \Pr_D[f(x) \neq 0] & = & \Pr_D[x_1 = 0]\Pr_{D_0}[f_0(x) \neq 0] + \Pr_D[x_1 = 1]
 \Pr_{D_1}[f_1(x) \neq 0]\\
 & \geq & \left(\frac{1}{1 + \alpha} \right)^{\log(t)}
 \end{eqnarray*}
 This completes the induction.
 \end{proof}

Using Lemma~\ref{lem:non-zero-prob}, we can now show that step 3 in algorithm~\ref{alg:learn_poly} correctly identifies all the important monomials (monomials of
low-degree with non-zero coefficients in $f$). 

\begin{lemma}\label{lem:all-important}
Suppose $S \subseteq [n]$, such that $f_S(x)$ has a monomial of degree at most
$d - |S|$ with non-zero coefficient. Then, 
\[ \Pr_{D_{\bar{S}}}[f_S(x_{\bar{S}}) \neq 0] \geq \left(\frac{1}{1 + \alpha} \right)^{d -
|S| + \log(t)} \]
\end{lemma}
%
%
%\begin{proof}
\begin{proof} 
Note that, since $D$ is a locally $\alpha$-smooth distribution, $D_{\bar{S}}$
is also a locally $\alpha$-smooth distribution (see Fact~\ref{fact:smooth}). Let
$S^\prime$ be a subset of $\bar{S}$, such that $\xi_{S^\prime}(x_{\bar{S}})$ is the
smallest degree monomial in $f_S(x_{\bar{S}})$ with non-zero coefficient. Then, since
$D_{\bar{S}}$ is locally $\alpha$-smooth,
$\Pr_{D_{\bar{S}}}[\xi_{S^\prime}(x_{\bar{S}})=1] \geq
(1/(1+\alpha))^{|S^\prime|} \geq (1/(1+\alpha))^{d - |S|}$.

Now, the conditional distribution $(D_{\bar{S}}| \xi_{S^\prime}(x_{\bar{S}}) = 1)$ is not necessarily
locally $\alpha$-smooth, but the marginal distribution with respect to variables
$(\overline{S \cup S^\prime})$,
$(D_{\bar{S}}|(\xi_{S^\prime}(x_{\bar{S}})=1))_{(\overline{S \cup S^\prime})}$, is
indeed locally $\alpha$-smooth (see Fact~\ref{fact:smooth}).  Let
$f^{S^{\prime}}_{S}(x_{(\overline{S \cup S^\prime})})$ be the polynomial obtained from
$f_S$ by setting $x_i =1 $ for each $i \in S^\prime$. Note that the constant
term of $f^{S^\prime}_{S}$ is non-zero and it is $t$-sparse, and is only defined
on the variables in $- (S \cup S^\prime)$. Hence, by applying
Lemma~\ref{lem:non-zero-prob} to $f^{S^\prime}_{S}$ and the marginal (w.r.t the
variables $\bar{S^\prime}$) of the conditional distribution
$(D_{\bar{S}}|\xi_{S^\prime}(x_{\bar{S}}) = 1)$, \ie
$(D_{\bar{S}}|\xi_{S^\prime}(x)=1)_{(\overline{S \cup S^\prime})}$, we get the required result. 
\end{proof}

%\end{proof}

Next, we show the following simple lemma~(proof is in Section~\ref{app:learn_poly}) that will allow us to conclude that
step 3 of the algorithm never adds too many terms.

\begin{lemma} \label{lem:not-too-many} If each term of $f_S$ has degree at least
$d^\prime$, then the probability that $f_S(x) \neq 0$ is at most $t
(\alpha/(1+\alpha))^{d^\prime}$.  \end{lemma}
%
%
% UNCOMMENT FOR LONG VERSION
 \begin{proof} Note that each monomial of $f_S$ has degree at least $d^\prime$.
 Under any locally $\alpha$-smooth distribution, the probability that a monomial of
 degree $d^\prime$ is not-zero is at most $(\alpha/(1+ \alpha))^{d^\prime}$ (see
 Fact~\ref{fact:smooth}). Since, $D_{\bar{S}}$ is a locally $\alpha$-smooth distribution, by
 a simple union bound we get the required result.  \end{proof}

%% USE THESE LINES FOR LONG VERSION
% Now in order to get an $\epsilon$-approximation in terms of squared error, using
% Lemma~\ref{lem:truncate}, it is clear that it suffices to choose $d =
% \log(4t^3B^2/\epsilon)/\log((1+\alpha)/\alpha)$, and consider the truncation
% $\alpha$. For this value of $d$, if $\theta$ is set to 
%% USE THESE TWO LINES FOR STOC VERSION
Now in order to get an $\epsilon$-approximation in terms of squared error, using
the argument about truncation, it suffices to choose $d =
\log(4t^3B^2/\epsilon)/\log((1+\alpha)/\alpha)$, and consider the truncation
$\alpha$. For this value of $d$, if $\theta$ is set to \\
$1/(4t^3B^2)^{2 \log(1+\alpha)/\log((1+\alpha)/\alpha)}$, using
Lemma~\ref{lem:all-important}, we are sure that all the monomials in $f$ of
degree at most $d$ that have non-zero coefficients are identified in step 3 of
the algorithm. Note that $\theta$ is still inverse polynomial in
$(ntB/\epsilon)^{\alpha}$. 

Finally, we note that if $d^\prime$ is set to
$\log(2t/\theta)/\log((1+\alpha)/\alpha)$, then for any subset, $S$, if the
monomial with the least degree in $f_S$, has degree at least $d^\prime$, then
$\Pr_{D_{\bar{S}}}[f_S(x) \neq 0] \leq \theta/2$. In particular, this means that if a
set, $S$, with $|S| \leq d$, is such that the smallest monomial, $\xi_T(x)$ in
$f$ for which $S \subseteq T$, is such that $|T| \geq d + d^\prime$, then $S$
will never be added to $\mc{S}$ by the algorithms. The fact that this
probability was $\theta/2$ (instead of exactly $\theta$), means that sampling
can be used carry out the test in the algorithm to reasonable accuracy. Finally,
observe that $ t 2^{d + d^\prime}$ is still polynomial in
$(ntB/\epsilon)^{\alpha}$. Thus, the total number of sets added in $\mc{S}$, can
never be more than polynomially many.  \medskip\\
%
% UNCOMMENT FOR LONG VERSION
 \noindent{\bf Generalization} The generalization argument is pretty standard and
 so we just present an outline. First, we observe that it is fine to discretize
 real numbers to some $\Delta$, where $\Delta$ is inverse polynomial in
 $(ntB/\epsilon)^{\alpha}$, without blowing up the squared loss. Now, the
 regression in the algorithm requires that the sum of absolute values of the
 coefficients of the polynomial, $h$, be at most $tB$. Thus, we can view this as
 distributing $tB/\Delta$ blocks over $2^n$ possible coefficients (in fact the
 number of coefficients is smaller). The total number of such discretized
 polynomials is at most $2^{\poly((ntB/\epsilon)^{\alpha})}$. Thus, it suffices
 to minimize the squared error on a (reasonably large) sample.

%% file: DNF_unif_app.tex
\begin{proof}[Proof of Claim~\ref{claim:not-too-many-DNF}]
For a set $S$, denote $\Vert f_S \Vert^2 = \sum_{T \supseteq S} \hat{f}^2(T)$.
If $S \in {\mathcal S}$, then we know that $\Vert f_S \Vert^2 \ge \theta^2$.
Thus, it follows that $|{\mathcal S}| \leq (1 /{\theta^2}) \sum_{S: |S| \le d}
\Vert f_S \Vert^2$.  Now, $\sum_{S: |S| \le d} \Vert f_S \Vert^2 = \sum_{d' =
1}^d \sum_{S: |S| = d'} \Vert f_S \Vert^2$. We will show that for each $d'$,
$\sum_{S: |S| = d'} \Vert f_S \Vert^2 \le ns 2^{O(d' \log(d'))}$. Then,
\begin{align*}
\sum_{S: |S| = d'} \Vert f_S \Vert^2 &=  \sum_{S: |S| = d'} \sum_{T \supseteq S} \hat{f}^2(T)\\
& = \sum_{T: |T| \ge d'} {|T| \choose d'} \hat{f}^2(T)\\
& = \sum_{t = d'}^n \sum_{T: |T|=t} {t \choose d'} \hat{f}^2(T)
\end{align*}
For a given $t$, consider the inner summation $\sum_{T: |T|=t} {t \choose d'}
\hat{f}^2(T)$. We will first apply Fact 1, to say that $f$ is close to some
$g_t$ which is an $l_t$-term $\DNF$ formula (we will define the value of
$l_t$ shortly). Hence, we get that,
\begin{align*}
\sum_{T: |T|=t} {t \choose d'} \hat{f}^2(T) & \le {t \choose d'} \sum_{T:
|T|=t}  \hat{g}_t^2(T) + \frac {4s} {2^{l_t}}
\end{align*}
Next we use Fact 2 to claim that $\sum_{T: |T|=t} \hat{g}_t^2(T) \le 2^{\frac {-t} {10 l_t}}$. So we have
\begin{align*}
\sum_{T: |T|=t} {t \choose d'} \hat{f}^2(T) & \le  {t \choose d'} \left(2^{\frac
{-t} {10 l_t}} + \frac {4s} {2^{l_t}}\right)
\end{align*}
Setting $l_t = C \sqrt{t}$ and differentiating, we get that the term is
maximized at $t = O({d'}^2)$ and the maximum value is $s2^{O(d' \log d')}$.
Since there are at most $n$ such terms we get that $\sum_{S: |S| = d'} ||f_S||^2
= ns2^{O(d' \log d')}$. Finally, we have
\[
\sum_{S: |S| \le d} ||f_S||^2 \le \sum_{d'=1}^d ns2^{O(d' \log(d'))} = nds
2^{O(d \log d)}.
\]
\end{proof}

%% file: dt_intro.tex
In this section, we present two algorithms for learning decision trees.  Section
\ref{sec:low-degree}, shows that $O(\log(n))$-depth decision trees can be
efficiently learned under locally $\alpha$-smooth distributions, for constant
$\alpha$.  This result is actually a special case of the result in
Section~\ref{sec:learn_poly}, but we present it separately because it is
somewhat simpler. In Section~\ref{sec:DT_unif}, we show that the class of
polynomial-size decision trees can be learned under the uniform distribution.
This algorithm is extended to product distributions; this fact is shown in
Section~\ref{sec:DT_prod}. Section~\ref{sec:rcn} shows that these algorithms can
be made to work under random classification noise.

%% file: low_degree.tex
In this section, the domain is assumed to be $\moo^n$.  Let $f$ be a function
that can be represented as a depth $d$ decision tree with $t$ leaves (note that
$t \leq 2^d$).  We are mostly interested in the case when $d =
O(\log(n))$,~\footnote{When $d = O(\log(n))$, whether we consider the domain to
be $\moo^n$ or $\zo^n$ is unimportant, since the sparsity of polynomial is
preserved (up to polynomial factors) when going from one domain to the other.}
since our algorithms run in time polynomial in $2^d$.  By slightly abusing
notation, let $f$ also denote the polynomial representing the decision tree,
\[ f(x) = \sum_{S \subseteq [n]} \hat{f}(S) \chi_S(x) \]
Also, we stick to the notation $\hat{f}(S)$ as the co-efficient of $\chi_S(x)$
even though we will consider distributions that are not uniform over $\moo^n$.
Thus, the coefficients cannot be interpreted as a Fourier transform. The
polynomial $f$ has degree $d$. Also, it is the case that $L_0(f) \leq t 2^d$,
$L_1(f) \leq t$, and $L_2(f) = 1$. (These facts hold even when the distribution
is not uniform and are standard. The reader is referred to~\citep{Man94}.
However, when the distribution is not uniform, it is not the case that $\E_{x
\sim D}[f(x)^2] = L_2(f)$.) More importantly, it is also the case that if
$\hat{f}(S) \neq 0$, then $|\hat{f}(S)| \geq 1/2^{d}$.  We will use this fact
effectively in showing that depth-$d$ decision trees can be efficiently learned
under locally $\alpha$-smooth distributions, when $d = O(\log(n))$ and $\alpha$
constant.

\begin{figure}[!t]
\begin{center}
\fbox{
\begin{minipage}{0.95 \columnwidth}
{\bf Algorithm}: \textsc{Learning log-depth decision trees} \smallskip \\
{\bf Input}: $d$ (depth of DT), $\alpha$, $\EX(f, D)$, (local)-$\MQ(f)$
\begin{enumerate}
\item {\bf let} $\mc{S} = \{ \emptyset \}$; $\theta = (1 + \alpha)^{-d -1}$
\item {\bf for} $i = 1, \ldots, d$
\begin{enumerate}
\item For every $S^\prime \in \mc{S}$, $|S^\prime| = i-1$ and for every $j \in
[n] \setminus S^\prime$
\begin{enumerate}
%%%
\item Let $S = S^\prime \cup \{j \}$
%%%
\item (Non-Zero Test) If $\Pr_{x \sim D}[f_S(x) \neq 0] \geq \theta$, then
$\mc{S} = \mc{S} \cup \{ S \}$
%%%
\end{enumerate}
\end{enumerate}
\item Let polynomial $h(x)$ over terms in $\mc{S}$ be obtained by minimizing
$\E_{x \sim \D}[(f(x) - h(x))^2]$, constrained by $\sum_{S} |\hat{h}(S)| \leq
t$.  
\end{enumerate}
{\bf Output}: $\sign(h(x))$.
\end{minipage}
}
\caption{Algorithm: Learning log-depth decision trees\label{alg:poly-smooth}}
\end{center}
\end{figure}

The algorithm (see Fig.~\ref{alg:poly-smooth}) finds all the monomials that are
relevant. Call a subset $S \subseteq [n]$ maximal for $f$, if $\hat{f}(S) \neq
0$ and for all $T \supsetneq S$, $\hat{f}(T) = 0$. We prove the following two
crucial points:
\begin{enumerate}
\item If $S$ is maximal for $f$, all subsets of $S$ pass the non-zero test (Step
2.(a).ii. of the algorithm).
\item If a set $T$ is not a subset of some $S$ that is maximal for $f$, $T$
fails the non-zero test (Step 2.(a).ii). 
\end{enumerate}
The above points can be used to prove two facts. First, that the relevant
monomials can be found by building up from the empty set (since all subsets of
maximal monomials pass the test). And second, that the total number of subsets
that are added into the set of important monomials (in Step 2) is at most $t
2^{d}$.  This bounds the running time of the algorithm.  \medskip 

\noindent{\bf Non-Zero Test}: In step 2.(a).ii. of the algorithm
(Fig.~\ref{alg:poly-smooth}), we check the following, which we call as the
\emph{non-zero} test: $\Pr_{x \sim D}[f_S(x) \neq 0]$. Recall, that
\[ f_S(x) = \sum_{T \supseteq S} \hat{f}(T) \chi_{T \setminus S}(x) \]
Note that in fact, $f_S$ is a function of $x_{\bar{S}}$. Let $x_{\bar{S}}$ be fixed. Let
$U_S$ denote the uniform distribution over $x_S$. Observe that $f(x) = f_{-S}(x)
+ \chi_{S}(x_S)f_S(x_{\bar{S}})$. Now each monomial in $f_{-S}$ is missing at least
one variable from $x_S$ (otherwise it would have been in $f_S$). Thus,
$f_S(x_{\bar{S}}) = \E_{x_S \sim U_S}[f(x_Sx_{\bar{S}})]$. Thus, if $x$ were a point drawn
from the example oracle, $\EX(f, D)$, $f_{S}(x_{\bar{S}})$ can be computed using
$|S|$-local membership queries. Now, $\Pr_{x \sim D}[f_S(x) \neq 0]$ can be
estimated very accurately by sampling. We ignore the analysis that employs
standard Chernoff bounds and assume that we can perform the test in Step
2.(a).ii. with perfect accuracy. We prove the following:

\begin{theorem} \label{thm:main-low-depth} The class of depth-$O(\log(n))$
decision trees is learnable using $O(\log(n/\epsilon))$-local membership queries
under the class of locally $\alpha$-smooth distributions, for constant $\alpha$, in time
that is polynomial in $n, 1/\epsilon, 1/\delta$.  \end{theorem}

The main tools required to prove Theorem~\ref{thm:main-low-depth} are
Lemmas~\ref{lem:main-tool1} and \ref{lem:main-tool2}. Lemma~\ref{lem:main-tool1}
shows that if some subset $S$ satisfies $\Pr[f_S(x) \neq 0] \geq \theta$, then
there exists some $T \supseteq S$ such that $\hat{f}(T) \neq 0$.
Lemma~\ref{lem:main-tool2} can be used to show that if $S \subseteq T$, where
$\hat{f}(T) \neq 0$ and $T$ is maximal for $f$, then $\Pr_{x \sim D}[f(x) \neq
0] \geq \theta$. Note that what Lemma~\ref{lem:main-tool2} actually proves is
that if $\Pr_{x \sim D}[f_{S}(x) \neq 0] \leq \theta$, then $\Pr_{x \sim D}[f_{S
\cup \{i\}}(x) \neq 0] \leq \theta (1 + \alpha)$, for any $i \not\in S$.  Now,
if $T$ is maximal, and $S \subseteq T$, this would imply that $\Pr_{x \sim
D}[f_T(x) \neq 0] \leq \theta (1 + \alpha)^{|T|-|S|}$. But, if $T$ is maximal,
the polynomial $f_T(x)$ is a non-zero constant polynomial. Thus, if $\theta$ is
chosen to be smaller than $1/(1 + \alpha)^d$, where $d = O(\log(n))$ is a bound
on the depth of the decision tree being learned, the algorithm finds all the
important monomials. Once all the important monomials have been found, the best
polynomial can easily be obtained, for example by regression. Alternatively, one
could also solve a system of linear equations on the monomials and this actually
learns the decision tree \emph{exactly}.\footnote{By this we mean, that the
output hypothesis is such that $\Pr_{x \sim D}[h(x) \neq f(x)] = 0$, except with
some small probability.}

\begin{lemma}\label{lem:main-tool1}
For any set $S \subseteq [n]$, if for all $T \supseteq S$, $\hat{f}(T) = 0$,
then $\Pr_{x \sim D}[f_S(x) \neq 0] = 0$.
\end{lemma}
\begin{proof} The polynomial $f_S(x) \equiv 0$. \end{proof}

\begin{lemma} \label{lem:main-tool2}
\label{lem:set-failure}
Let $D$ be any locally $\alpha$-smooth distribution, then for any set $S$, and any $i
\not\in S$, $\Pr_{x\sim D}[f_{S \cup \{i\}}(x) \neq 0] \geq (1/(1+\alpha))
\Pr_{x \sim D}[f_S(x) \neq 0]$.
\end{lemma}
\begin{proof}
Recall, that $f_S(x) = \sum_{T \supseteq S} \hat{f}(T) \chi_{T \setminus S}(x)$.
Also, note that $f_S$ is only a function of the variables, $x_{\bar{S}}$. 
Observe that, 
\begin{align*}
f_S(x_{\bar{S}}) &= \sum_{\substack{ T \supseteq S \\ i \not\in T}} \hat{f}(T)
\chi_{T \setminus S}(x_{\bar{S}}) + x_i f_{S \cup \{i \}}(x_{-(S \cup \{i\})}) 
\end{align*}
Note that $D_{\bar{S}}$, the marginal distribution, is also locally $\alpha$-smooth (see
Fact~\ref{fact:smooth}). Let $(D_{\bar{S}} | f_{S \cup \{i\}}(x_{-(S \cup \{i\})}) =
1)$ be the conditional distribution, given $f_{S \cup \{i \}}(x_{-(S \cup \{i\})})
= 1$. Then, the marginal distribution, $(D_{\bar{S}}|f_{S \cup \{i\}}(x_{-(S \cup
\{i\})}))_{i}$ is a distribution on only one variable, $x_i$, but is also
locally $\alpha$-smooth (Fact~\ref{fact:smooth}). But, given that $f_{S \cup
\{i\}}(x_{-(S \cup \{i\})}) \neq 0$, for one of the two values, $x_i =1 $ or
$x_i = -1$, it must be the case that $f_S(x_{\bar{S}}) \neq 0$. Since, the
distribution $(D_{\bar{S}}|f_{S \cup \{i \}}(x_{-(S \cup \{i\})}) = 1)_i$ is
locally $\alpha$-smooth, the proof of the Lemma follows from Fact~\ref{fact:smooth}.
(Note that $\Pr_{x\sim D}[f_S(x) \neq 0] = \Pr_{x_{\bar{S}} \sim
D_{\bar{S}}}[f_S(x_{\bar{S}})
\neq 0]$, since $f_S$ does not depend on the variables $x_S$.)
\end{proof}

%% file: DT_unif.tex
% LONG VERSION
In this section, we present an algorithm for learning $t$-leaf decision trees
(of arbitrary depth) under the uniform distribution. Although, the uniform
distribution is a special case of product distributions considered in Section
\ref{sec:DT_prod}, the exposition is simpler and conveys the high-level ideas
better. 
% 
%% STOC VERSION
%In this section, we present an algorithm for learning $t$-leaf decision trees
%(of arbitrary depth) under the uniform distribution. We outline how this can be
%extended to product distributions. The full details are provided in the long
%version of this paper.

%%
We use standard results from Fourier analysis; \citet{KushilevitzMansour:93} proved the following
useful properties of the Fourier spectrum of decision trees. Let $f$ be a
function that is represented by a $t$-leaf decision tree, then:
\begin{enumerate}
\item For any set $S \subseteq [n]$, $|\hat{f}(S)| \leq t/2^{|S|}$. 
\item $L_1(f) = \sum_{S \subseteq [n]} |\hat{f}(S)| \leq t$.
\end{enumerate}
Using the above relations, we can immediately prove the following useful (and
well-known) fact. 

\begin{fact} Suppose $f$ is boolean function that is represented by $t$-leaf
decision tree.  Then, for any $\tau > 0$, 
%% REMOVE THE \\ in the above line for the LONG VERSION
$\sum_{S, |S| \geq \log(t^2/\tau)} \hat{f}(S)^2 \leq \tau$ \label{fact:conc-dt}
\end{fact}
%
% COMMENT OUT FOR STOC VERSION
\begin{proof} 
Consider, 
\begin{align*}
\sum_{S, |S| \geq \log(t^2/\tau)} \hat{f}(S)^2 &\leq \max_{S, |S| \geq
\log(t^2/\tau)} |\hat{f}(S)| \cdot \left( \sum_{T, |T| \geq \log(t^2/\tau)}
|\hat{f}(S)| \right) \\
&\leq t \cdot (\tau / t^2) \cdot L_1(f) \leq \tau
\end{align*}
\end{proof}

The algorithm in Figure \ref{alg:unif-dt} learns $t$-leaf decision trees under
the uniform distribution. For simplicity of presentation, we assume that the
expectations used in the algorithm and also the Fourier coefficients can be
computed \emph{exactly}. It is easy to see that using standard applications of
Chernoff-Hoeffding bounds, the guarantees of the algorithm hold even when the
expectations and values of the Fourier coefficients can only be computed
approximately. The main step in Algorithm~\ref{alg:unif-dt} that requires some
explanation is how to compute the quantity $\E_{x \sim U}[f_S(x)^2]$ to check if
it is greater than $\theta^2$. We refer to this as the $L_2$ Test. \smallskip

\noindent{\bf $L_2$ Test}: Let $x \in \moo^n$, and recall that for $S \subseteq
[n]$, $f_S(x) = \sum_{T \supseteq S} \hat{f}(T) \chi_{T \setminus S}(x)$, and
that this can be computed by using the fact that, 
\[ f_S(x_{\bar{S}}) = \E_{x_S \sim U_S} [\chi_{S}(x) f(x)] \]
Given a point $(x, f(x))$, we observe that the expectation $\E_{x_S \sim
U_S}[f(x) \chi_S(x)]$ can be computed using $2^{|S|}$, $|S|$-local membership
queries with respect to $x$ (only the bits in $S$ need to be flipped). The
quantity $\E_{x \sim U}[f_S(x)^2]$ can thus be computed easily using only
$|S|$-local membership queries and taking a sample from $\EX(f, D)$.  \smallskip
\begin{figure}[!t]
\begin{center}
\fbox{
\begin{minipage}{0.95 \columnwidth}
{\bf Algorithm}: \textsc{Learning Decision Trees} \smallskip \\
~~{\bf inputs}: $d$, $\theta$, oracles $\EX(f, U)$, (local)-$\MQ(f)$ 
%~~\texttt{\#} $f$ is a size-$s$ decision tree \\
%~~\texttt{\#} $U$ is the uniform distribution over $\moo^n$. 
%
\begin{enumerate}
\item {\bf let} $\mc{S} = \{ \emptyset \}$
\item {\bf for} $i = 1, \ldots, d$
\begin{enumerate}
\item {\bf for} every $S^\prime \in \mc{S}$, $|S^\prime| = i-1$ and for every $j \in
[n] \setminus S^\prime$
\begin{enumerate}
\item {\bf let} $S = S^\prime \cup \{j \}$
\item ($L_2$ Test){ \bf if} $\E_{x \sim U}[f_S(x)^2] > \theta^2$, {\bf then} $\mc{S} = \mc{S} \cup \{ S
\}$
\end{enumerate}
\end{enumerate}
\item {\bf let} $h(x) = \sum_{S \in \mc{S}} \hat{f}(S) \chi_S(x)$ \end{enumerate}
~~{\bf output}: $\mathrm{sign}(h(x))$
\end{minipage}
}
\caption{Algorithm: Learning Decision Trees under the Uniform
Distribution\label{alg:unif-dt}}
\end{center}
\end{figure}

\noindent {\bf High-Level Overview of Proof}: Fact \ref{fact:conc-dt} showed
that the Fourier mass (sum of squares of the Fourier coefficients) of  $t$-leaf
decision trees is concentrated on low degree terms. Parseval's identity implies
that this is sufficient to construct a polynomial, $h(x)$, that is a good
$\ell_2$ approximation to the decision tree, $f$, i.e. $\E_{x \sim U}[(h(x) -
f(x))^2] \leq \epsilon$. Also, \citet{KushilevitzMansour:93} showed that
since $L_1(f)$ is bounded, most of the Fourier mass is concentrated on a small
(polynomially many) number of terms.

The main insight here is that, these terms on which most of the Fourier mass is
concentrated, can be identified using only $O(\log(n))$-local membership
queries.  It is relatively easy to see that any coefficient for which
$|\hat{f}(S)| \geq \theta$ will be identified correctly by the test in line
2.(a).ii. (Figure \ref{alg:unif-dt}). We show that the quantity $|\mc{S}|$ never
grows too large.  To show this, we prove that if any coefficient is inserted in
$\mc{S}$ in line 2.(a).ii, it must be a subset of some coefficient of large
magnitude. This follows quite easily by observing that $\E_{x \sim U}[f_S(x)^2]
= \sum_{T \supseteq S} \hat{f}(T)^2$ and using the fact that $L_1(f)$ is
bounded.

%% UNCOMMENT FOR LONG VERSION
The rest of the section is devoted to a formal proof of the above overview. 
%
%% UNCOMMENT FOR STOC VERSION
%The rest of the section outlines a proof of the above overview. Full details are
%provided in the longer version.

\begin{claim}
\label{claim:T-in-S}
Suppose that $S$ is such that $|\hat{f}(S)| \geq \theta$ and $|S| \leq d$, then
$S \in \mc{S}$. 
\end{claim}
% COMMENT OUT PROOF FOR STOC VERSION
\begin{proof}
First observe that for any subset $S^\prime \subseteq S$, it holds that
$\E[f_{S^\prime}(x)^2] \geq \theta^2$. This follows immediately by observing
that 
\begin{align*}
\E[f_{S^\prime}(x)^2] &= \sum_{T \supseteq S^\prime} \hat{f}(T)^2 \geq
\hat{f}(S)^2 \geq \theta^2
\end{align*}
% %
It follows by a simple induction argument that at iteration $i$, $\mc{S}$
contains every subset of $S$ of size at most $i$, for which $\E[f_{S}(x)^2] \geq
\theta^2$. And, hence $S \in \mc{S}$.
\end{proof}

\begin{claim}
\label{claim:T-in-S-mapsto-large-coefficient}
If $S \in \mc{S}$, then there exists a $S^\prime \supseteq S$ such that
$\hat{f}(S^\prime) \geq \theta^2/t$. 
\end{claim}
%% COMMENT OUT PROOF FOR STOC VERSION
\begin{proof} Since $S \in \mc{S}$, we know that $\E[f_S(x)^2] = \sum_{T
\supseteq S} \hat{f}(S)^2 \geq \theta^2$. But observe that,
\begin{align*}
\sum_{T \supseteq S} \hat{f}(T)^2 \leq \left(\sum_{T \supseteq S}
|\hat{f}(T)| \right) \cdot \max_{T \supseteq S} |\hat{f}(T)|
\end{align*}
% %
The above inequality simply states the fact that  $L_2(f_S) \leq L_1(f_S)
L_\infty(f_S)$. Since $f$ is a $t$-leaf decision tree, $\sum_{T \supseteq
S}|\hat{f}(T)| \leq L_1(f) \leq t$. The claim now follows immediately.
\end{proof}

Using the above claims, it is easy to show our main theorem.

\begin{theorem} \label{thm:unif-dt} Algorithm in Fig. \ref{alg:unif-dt} run with
parameters $d = \log(2 t^2/\epsilon)$ and $\theta = \epsilon/(2t)$, outputs a
hypothesis, $\mathrm{sign}(h(x))$, where $\err_U(\mathrm{sign}(h(x), f) \leq
\epsilon$. The running time is $\mathrm{poly}(t, n, 1/\epsilon)$ and the
algorithm only makes $\log(2 t^2/\epsilon)$-local queries to the membership
oracle $\MQ(f)$. \end{theorem}
\begin{proof}
First, we recall that for a $t$-leaf decision tree, $|\hat{f}(S)| \leq
t/2^{|S|}$~\citep{KushilevitzMansour:93}.  Thus, if $|\hat{f}(S)| \geq \theta^2/t$, then
$|S| \leq  2 \log(t/\theta)$.  Using Parseval's identity (see Section
\ref{sec:notation}), we know that the number of Fourier coefficients that have
magnitude greater than $\theta^2/t$ is at most $t^4/\theta^2$. 

Consider the set $\mc{S}$ constructed by the algorithm (Fig. \ref{alg:unif-dt})
at the end of $d$ iterations. If $S \in \mc{S}$, then there must exist some $T
\supseteq S$ such that $|\hat{f}(T)| \geq \theta^2/t$ (Claim
\ref{claim:T-in-S-mapsto-large-coefficient}). But there can be at most
$t^2/\theta^4$ such terms and each is of size at most $2\log(t/\theta)$. Hence,
the $|\mc{S}| \leq (t^2/ \theta^4) 2^{2 \log(t/\theta)} = t^4/\theta^6$. 

For any coefficient, such that $|\hat{f}(S)| \geq \theta$, it must be that $|S|
\leq \log(t/\theta) \leq d$. Claim \ref{claim:T-in-S} shows that all such
coefficients are included in $\mc{S}$. Thus, $\max_{S \not\in \mc{S}}
|\hat{f}(S)| \leq \theta$. Hence, $\sum_{S \not\in \mc{S}} \hat{f}(S)^2 \leq
\sum_{S \not\in \mc{S}} |\hat{f}(S)| \cdot \max_{S \not\in\mc{S}} |\hat{f}(S)|
\leq L_1(f) \cdot \theta \leq \theta t$. But $\E[(h(x) - f(x))^2] = \sum_{S \not\in
\mc{S}} \hat{f}(S)^2$ and also notice that $\Pr_{x \sim U}[\sign(h(x)) \neq
f(x)] \leq \E_{x \sim U}[(h(x) - f(x))^2]$ (since $f(x)$ only takes values $\pm
1$).  \end{proof}

%% file: DT_prod.tex
In this section, we prove that the class of $t$-leaf decision trees can be
learned under the class of product distributions, where each bit has mean
bounded away from $-1$ and $1$. Let $\mu = (\mu_1, \ldots, \mu_n)$ denote a
product distribution over $X = \moo^n$, where $\E_{x \sim \mu}[x_i] = \mu_i \in
[-1 + 2c, 1 - 2c]$, for some constant $c \in (0, 1/2]$. We use Fourier analysis
using the modified basis for the product distribution. We begin by introducing
required notation for using Fourier techniques. \medskip

\noindent {\bf Fourier Analysis over $\mu$}: Let $\mu = (\mu_1, \ldots, \mu_n)$
be the product distribution over $X = \moo^n$, where $\E_{x \sim \mu}[x_i] =
\mu_i$. Define, 
\[ \chi^\mu_S(x) = \prod_{i \in S} \frac{x_i - \mu_i}{\sqrt{ (1 -
\mu_i^2)}}.\] 
Then, it is easy to observe that for any two sets $S_1 \neq S_2$, $\E_{x \sim
\mu}[\chi^\mu_{S_1}(x) \chi^\mu_{S_2}] = 0$ and that, for any set $S$, $\E_{x
\sim \mu}[\chi^\mu_S(x)^2] = 1$. Thus, the set of functions $\langle
\chi^\mu_S(x) \rangle_{S \subseteq [n]}$ forms an orthonormal basis for
functions defined on $\{-1, 1\}^n$ under the distribution $\mu$. For any
function $f: \moo^n \rightarrow \reals$, the Fourier coefficients under
distribution $\mu$ are defined as $\hat{f}^\mu(S) = \E_{x \sim \mu}[f(x)
\chi^\mu_S(x)]$. 
The following is Parseval's identity in this basis:
\begin{align}
\label{eq:parseval} \E_{x \sim \mu}[f(x)^2] = \sum_{S \subseteq [n]}
\hat{f}^\mu(S)^2
\end{align}
In particular, when $f$ is a boolean function, i.e. with range $\{-1, 1\}$, the
sum of Fourier coefficients is $1$.

Let $L^\mu_1(f) = \sum_{S \subseteq [n]} |\hat{f}^\mu(S)|$, $L^\mu_2(f) =
\sum_{S \subseteq [n]} \hat{f}^\mu(S)^2$ and $L^\mu_\infty(f) = \max_{S
\subseteq [n]} |\hat{f}^{\mu}(S)|$ denote the $1$, $2$ and $\infty$ norm of the
Fourier spectrum under distribution $\mu$. Also let $L^\mu_0(f) = |\{ S ~|~
\hat{f}^\mu(S) \neq 0 \} |$ denote the number of non-zero Fourier coefficients
of $f$. We will frequently use the following useful observations: 
\begin{enumerate}
\item $L^\mu_2(f) \leq L^\mu_1(f) \cdot L^\mu_\infty(f)$
\item $L^\mu_2(f) \leq L^\mu_0(f) \cdot (L^\mu_\infty(f))^2$
\end{enumerate}

\subsubsection{Decision Tree Learning Algorithm}

\begin{figure}[!ht]
\begin{center}
\fbox{
\begin{minipage}{0.95 \columnwidth}
{\bf Algorithm}: \textsc{Learning Decision Trees} \smallskip \\
~~{\bf inputs}: $d$, $\theta$, oracles $\EX(f, \mu)$, $\MQ(f)$ \\
~~\texttt{\#} $f$ is a $t$-leaf decision tree \\
~~\texttt{\#} $\mu$ is a product distribution over $\moo^n$, $\mu_i \in [-1 +
2c, 1 - 2c]$ \\
\begin{enumerate}
\item {\bf let} $\mc{S} = \{ \emptyset \}$
\item {\bf for} $i = 1, \ldots, d$
\begin{enumerate}
\item {\bf for} every $S^\prime \in \mc{S}$, $|S^\prime| = i-1$ and for every $j \in
[n] \setminus S^\prime$
\begin{enumerate}
\item {\bf let} $S = S^\prime \cup \{j \}$
\item ($L_2$ Test) {\bf if} $\E_{x \sim \mu}[f_S(x)^2] > \theta^2$, {\bf then} $\mc{S} =
\mc{S} \cup \{ S \}$
\end{enumerate}
\end{enumerate}
\item {\bf let} $h(x) = \sum_{S \in \mc{S}} \hat{f}^\mu(S) \chi^\mu_S(x)$
\end{enumerate}
~~{\bf output}: $\mathrm{sign}(h(x))$
\end{minipage}
}
\caption{Algorithm: Learning Decision Trees under Product
Distributions\label{alg:prod-dt}}
\end{center}
\end{figure}

We present a high-level overview of our algorithm and a formal statement of the
main result, before providing full details. The Algorithm is described in Figure
\ref{alg:prod-dt}. \smallskip

\noindent {\bf Truncation}: We show that a $t$-leaf decision tree, when
truncated to logarithmic depth, is still a very good (inverse polynomially
close) approximation to the original decision tree. This observation can be used
to show that it suffices to identify low-degree (logarithmic) ``heavy" Fourier
coefficients of $f$, with respect to the distribution, $\mu$, and also that the
number of such terms is not too large (at most polynomial). Note that this is
not as simple as in the case of the uniform distribution, because it is not
straightforward to bound $L^\mu_1(f) = \sum_{S \subseteq [n]} |\hat{f}^\mu(S)|$.
(When $\mu$ is the uniform distribution, this is bounded by $t$.) Properties of
such truncated decision trees were also used by~\citet{Kalai09} in
the smoothed analysis setting. 

A $t$-leaf decision tree  can be though of as $t$ (not disjoint) paths from root
to leaves. A truncation of a decision tree at depth $d$, is a decision tree
where for each path of length more than $d$, only the prefix (from root) of
length $d$ is preserved.  Note that this may collapse several paths to the same
prefix, possibly reducing the number of leaves. A new leaf is added at the end
of this path and labeled arbitrarily as $-1$ or $+1$. 

For any function $g$, we denote by $\mc{S}^\mu_g$, the set of non-zero Fourier
coefficients of $g$, with respect to the product distribution, $\mu$, \ie
$\mc{S}^\mu_g = \{ T \subseteq [n] ~|~ \hat{g}(T) \neq 0 \}$. 

We prove two useful properties of the truncated decision trees with respect to
product distribution. These appear as formal statements in Lemmas
\ref{lem:DT-prod-truncate} and \ref{lem:DT-prod-nonzero-coeff}. Similar
observations were also used by~\citet{Kalai09} to prove learning of decision
trees in the smoothed analysis setting.
\begin{enumerate}
\item[(i)] Truncation at logarithmic depth is a good approximation (inverse
polynomial) to the
original decision tree.
\item[(ii)] The number of nonzero Fourier coefficients of the truncated decision
tree, $|\mc{S}^\mu_g|$ is small (polynomial).
\end{enumerate}

\begin{lemma}
\label{lem:DT-prod-truncate}
Let $f$ be a $t$-leaf decision tree, let $\mu$ be a product distribution over $X
= \moo^n$ such that $\mu_i \in  [-1 + 2c, 1 - 2c]$. Then for every $\tau > 0$,
there exists a $t$-leaf decision tree of depth at most
$\log(t/\tau)/\log(1/(1-c))$, such that $\Pr_{x \sim \mu}[g(x) \neq f(x)] \leq
\tau$
\end{lemma}
\begin{proof}
Let $g$ be the decision tree obtained by truncating $f$ at depth $d$. The new
leaves added at depth $d$ can be labeled arbitrarily. Now, the points $x$ for
which $g(x) \neq f(x)$ are precisely those, for which $g$ would lead to the
newly added leaf node at depth $d$. But since $\E_{x \sim \mu} [x_i] \in [-1 +
2c, 1 - 2c]$, the probability that a random point from $\mu$ reaches such a node
is at most $(1-c)^d$. The number of new leaf nodes added cannot be more than $t$
(since any truncation only reduces the number of leaves). Thus, $\Pr_{x \sim
\mu}[g(x) \neq f(x)] \leq t (1-c)^d$. When, $d = \log(t/\tau)/\log(1/(1-c))$ we get
the result.
\end{proof}

\begin{lemma}
\label{lem:DT-prod-nonzero-coeff}
Let $g$ be a decision tree of depth $d$ and $t$ leaves; then the number of
non-zero Fourier coefficients of $g$ is at most $t \cdot 2^d$ and each is of
size at most $d$.
\end{lemma}
\begin{proof}
We consider any path in $g$ from root to leaf, and let $P$ denote the subset of
indexes corresponding to the variable that occur in the path. First, we expand
decision tree $g$ as a polynomial. 
\begin{align*}
g(x) = \sum_{\mbox{path}~P} y_P \prod_{i \in P} \frac{1 + \sigma_{P, i} x_i}{2},
\end{align*}
where $\sigma_{P, i}$ is $+1$ or $-1$, depending on whether the path leading out
of node labeled $x_i$ on path $P$ was labeled $+1$ or $-1$, and $y_P$ is the
label of the leaf at the end of the path $P$.

The only nonzero coefficients in $g$ are of the from $\prod_{i \in T} x_i$ for
some $T \subseteq P$ for some path $P$. This also means that the only non-zero
Fourier coefficients can be those corresponding to such subsets. This is because
$\E_{x \sim \mu}[\chi^\mu_T(x) \prod_{i \in S}x_i] = 0$, unless $T \subseteq S$
(because $\mu$ is a product distribution).  Since the number of paths in $g$ is
at most $t$ and the length of each path is at most $d$, we get the required
result.
\end{proof}

\begin{lemma} 
\label{lem:DT-prod-sum-coeff}
Let $f$ be a $t$-leaf decision tree, let $g$ be a truncation of $f$ to depth
$\log(4t / \tau)/\log(1/(1-c))$. Then, 
\[\sum_{S, S \not\in \mc{S}^\mu_g } \hat{f}_\mu(S)^2 \leq \tau\]
\end{lemma}
\begin{proof}Let $g$ be a truncation of $f$ at depth $\log(4t/
\tau)/\log(1/(1-c))$. Let $\mc{S}^\mu_g$ denote the set of non-zero Fourier
coefficients of $g$ under distribution $\mu$. Using Lemma
\ref{lem:DT-prod-truncate}, we know that $\Pr_{x \sim \mu} [f(x) \neq g(x)] \leq
\tau/4$, hence $\E[(f(x) - g(x))^2] \leq \tau$.  Now, by Parseval's identity: 
\begin{align*}
\tau &\geq \E_{x \sim \mu}[(f(x) - g(x))^2] \\
&= \sum_{S \subseteq \mc{S}_g} (\hat{f}(S) - \hat{g}(S))^2 + \sum_{S
\not\in {S}_g} \hat{f}(S)^2 \\
&\geq \sum_{S \not\in {S}_g} \hat{f}(S)^2
\end{align*}
The proof is complete by observing that every coefficient $S \in \mc{S}^\mu_g$
satisfies $|S| \leq \log(4t/\tau)/\log(1/(1-c))$ by Lemma
\ref{lem:DT-prod-nonzero-coeff}.
\end{proof}

\noindent {\bf $L_2$ Test}:  As in the case of uniform distribution, we write
$f(x)$ as: 
\begin{align*} f(x) = f^\mu_{-S}(x) + \chi^\mu_{S}(x) f^\mu_S(x), \end{align*} 
where, $f^\mu_{-S}(x) = \sum_{T, S \not\subseteq T} \hat{f}^\mu(T)
\chi^\mu_T(x)$ and  $f^\mu_S(x) = \sum_{T \supseteq S} \hat{f}^\mu(T)
\chi^\mu_{T \setminus S}(x)$. 
Then as in the case of uniform distribution, $f_S(x) = f_S(x_{-S}) = \E_{x_S
\sim \mu_S}[f(x) \chi^\mu_S(x)]$, where now $x_S$ is drawn from the restriction
$\mu_S$ of the product distribution to the bits $x_S$. Note that for any given
point $x$, $f_S(x)$ can be computed easily using $2^{|S|}$ membership queries
that are $|S|$-local (since only the bits $x_S$ need to be changed).  We point
out that there is a subtle point in the case of product distributions. Recall
that $f_S(x) = \E_{x_S \sim  \mu_S}[ f(x)\chi^\mu_S(x)]$.  In the case when
$\mu$ is the uniform distribution, the parity functions, $\chi_S$ are $\{-1,
1\}$ valued, and so $f_S(x) \in [-1, 1]$. Thus, application of
Chernoff-Hoeffding bounds is straightforward.  In the case, of product
distributions the range of $\chi^\mu_S(x)$ can be  $[- \prod_{i \in S} ((1 -
|\mu_i|)/(\sqrt{1 - \mu_i^2})), \prod_{i \in S} ((1 + |\mu_i|)/(\sqrt{1 -
\mu_i^2}))]$. Since, we never consider sets $S$ that are larger than
$O(\log(n/\epsilon))$, the range of $f_S$ in our case is still polynomially
bounded and arbitrarily good (inverse polynomial) estimates to the true
expectation of $\E_{x \sim \mu}[f_S(x)^2]$ can be obtained by taking a sample and
applying Chernoff-Hoeffding bounds. Thus, to simplify the presentation, we
assume we can compute the expectation (in Line 2.a.ii in Fig. \ref{alg:prod-dt})
and the Fourier coefficients exactly.

Theorem \ref{thm:prod-dt} is the statement of the formal result about learning
decision trees under product distributions. The main ideas are similar to the
proof in the case of uniform distribution; but, the proof is more involved as
explained above. 

\begin{theorem} \label{thm:prod-dt}
Algorithm in Fig. \ref{alg:prod-dt} with parameters $\theta =
\sqrt{\epsilon/(2t(8t/\epsilon)^{1/\log(1/(1-c))})}$, \\ $d =
\log(8t/\epsilon)/\log(1/(1-c)) $, outputs a hypothesis $\sign(h(x))$, such that
$\err_\mu(\sign(h(x)), f) \leq \epsilon$.  The running time of the algorithm is
polynomial in $n$, $t$ and $1/\epsilon$ and the algorithm makes only
$O(\log(nt/\epsilon))$-local membership queries to the oracle $\MQ(f)$.
\end{theorem}

The rest of this section is devoted to the proof of Theorem~\ref{thm:prod-dt}.

\begin{claim}\label{claim:prod-all-heavy} If $S$ is such that $|\hat{f}^\mu(S)|
\geq \theta$ and $|S| \leq d$, then $S \in \mc{S}$.  \end{claim}
\begin{proof} This proof is thee same as the proof of Claim \ref{claim:T-in-S}.
\end{proof}

\begin{claim}\label{claim:prod-not-too-many}
If $S \in \mc{S}$, then there exists $S^\prime \supseteq S$, such that
$\hat{f}^\mu(S^\prime)^2 \geq (\theta^2/2)/(t \cdot
(8t/\theta^2)^{1/\log(1/(1-c))})$ and $|S^\prime| \leq
\log(8t/\theta^2)/\log(1/(1-c))$.
\end{claim}
\begin{proof}
Let $\tau = \theta^2/2$ and let $g^\prime$ be the decision tree obtained by
truncation of $f$ as described in Lemma \ref{lem:DT-prod-sum-coeff}. Then, by Lemma
\ref{lem:DT-prod-nonzero-coeff}, we know the depth of $g^\prime$ is
$\log(8t/\theta^2)/\log(1/(1-c))$ and that $\mc{S}^\mu_{g^\prime}$ is of size at
most $t \cdot 2^{\log(8t/\theta^2)/\log(1/(1-c))} = t \cdot
(8t/\theta^2)^{1/\log(1/(1-c))}$. Also, by Lemma \ref{lem:DT-prod-sum-coeff} we know
that $\sum_{T \not\in \mc{S}^\mu_{g^\prime}} \hat{f}^\mu(T)^2 \leq \theta^2/2$,
and hence if $S$ passes the $L_2$-test, i.e. $\sum_{T \supseteq S}
\hat{f}^\mu(T)^2 \geq \theta^2$, it must be that $\sum_{T \supseteq S, T \in
\mc{S}^\mu_{g^\prime}} \hat{f}^\mu(T)^2 \geq \theta^2/2$. Hence, there must be
some set $S^\prime$ of size at most $\log(8t/\theta^2)/\log(1/(1-c))$ for which
$\hat{f}^\mu(S^\prime)^2 \geq (\theta^2/2)/(t \cdot
(8t/\theta^2)^{1/\log(1/(1-c))})$.
\end{proof}

\begin{proof}[Proof of Theorem \ref{thm:prod-dt}] Let $g$ be the truncation of
the target decision tree, $f$, to depth $d$. Then using Lemma
\ref{lem:DT-prod-sum-coeff}, we know that $\sum_{S \not\in
\mc{S}^\mu_{g^\prime}} \hat{f}^\mu(S)^2 \leq \epsilon/2$. Now, every coefficient
in $S \in \mc{S}^\mu_g$ for which $|\hat{f}^\mu(S)| \geq \theta$ is in $\mc{S}$
(see Algorithm \ref{alg:prod-dt} and Claim \ref{claim:prod-all-heavy}).
$|\mc{S}^\mu_{g^\prime}| \leq t 2^d$. Tedious calculations show that $\sum_{S
\in \mc{S}^\mu_{g^\prime}, |\hat{f}(S)| < \theta} \hat{f}(S)^2 \leq t 2^d
\theta^2 \leq \epsilon/2$. Thus, $\sum_{S \in \mc{S}} \hat{f}(S)^2 \geq \sum_{S
\in \mc{S} \cap \mc{S}^\mu_{g^\prime}} \hat{f}(S)^2 \geq 1 - \epsilon$. This
implies by Parseval, that $\E_{x \sim \mu}[(h(x) - f(x))^2] \leq \epsilon$,
where $h(x)$ is as defined in Algorithm \ref{alg:prod-dt}.

The only thing remaining to show is that $|\mc{S}|$ always remains bounded by
$\poly(t, n, 1/\epsilon)$. This can be shown easily using Claim
\ref{claim:prod-not-too-many}, since if $S \in \mc{S}$, there exists $S^\prime
\supseteq S$, such that $|S^\prime| \leq \log(8t/\theta^2)/\log(1/(1-c))$ and
$\hat{f}^\mu(S)^2 \geq (\theta^2/2)/(t \cdot (8t/\theta^2)^{1/\log(1/(1-c))}$.
Thus, the magnitude of $\hat{f}(S^\prime)^2$ is at least $1/\poly(t, n,
1/\epsilon)$, so by Parseval there can be at most $\poly(t, n, 1/\epsilon)$.
Also the size of $|S^\prime|$ is $O(\log(tn/\epsilon))$, thus the total number
of irrelevant subsets added to $\mc{S}$ is at most $\poly(t, n, 1/\epsilon)$.
\end{proof}

%% file: rcn.tex
In this section, we show how the algorithms for learning decision trees can be
implemented even with access to a noisy oracle. The learning algorithm we use is
allowed queries to the membership oracle, $\MQ(f)$, therefore we consider a
persistent random noise model. An easy way to conceptualize this model is as
follows: Let $\zeta : \moo^n \rightarrow \moo$ be a function where for each $x
\in \moo^n$, the value of $\zeta(x) =1$ with probability $1 - \eta$ and $-1$ with
probability $\eta$, independently. Once this noise function, $\zeta$, has been
fixed, we assume that we have access to the function: $f^{\eta} = f \cdot
\zeta$, rather than the function $f$. We show how the tests mentioned in this
section can be implemented using $\EX(f^{\eta}, D)$ and $\MQ(f^{\eta})$, rather
than $\EX(f, D)$ and $\MQ(f)$. 

\subsubsection{Non-Zero Test}

Recall that we are interested in estimating $\Pr[f_S(x) \neq 0]$, where $S
\subseteq [n]$, and 
\begin{equation}
\label{eqn:compute-fseta}
 f_S(x) = \E_{x_S \sim U_S} [f(x) \chi_{S}(x)] 
\end{equation}
Instead, if we have access to $f^{\eta}$, we are able to compute, 
\[ f^{\eta}_S(x) = \E_{x_S \sim U_S} [f^{\eta}(x) \chi_{S}(x)] \]
Although, the random classification noise is persistent and fixed according to
$\zeta$, for the purpose of analysis it is easier to imagine that for each $x$,
$\zeta(x)$ is only determined when the algorithm makes a query for the point $x$
(or $x$ is drawn by $\EX(f^{\eta}, D)$). Lemma~\ref{lem:rcn_separation} allows
us to conclude that the test required in Section~\ref{sec:low-degree} can be
performed using access to $f^{\eta}$ instead of $f$. The lemma assumes that
$\zeta(x)$ is chosen independently, each time $x$ is queried, \ie the noise is
not \emph{persistent}. However, we show later that our algorithm queries each
example only once, so the noise may as well have been persistent.

\begin{lemma} The following are true:
\label{lem:rcn_separation}
\begin{enumerate}
\item $\Pr_{x, \zeta}[f^{\eta}_S(x) \neq 0] \geq (1 - p_0) + \frac{(2\eta -
1)^2 c_0}{2^{3|S|/2}} \Pr[f_S(x) \neq 0]$
\item $\Pr_{x, \zeta}[f^{\eta}_S(x) \neq 0] \leq (1 - p_0) + \Pr[f_S(x) \neq 0]$
\end{enumerate}
Here, $c_0$ is an absolute constant, $p_0$ depends only on $|S|$ and $\eta$. The
probability is taken over the choice of $x \sim D$ and choice of $\zeta$.
\end{lemma}
%
% COMMENT OUT PROOF FOR STOC VERSION
\begin{proof}
Note that $f_S(x) = \E_{x_S \sim U_S}[f(x) \chi_S(x)]$, and so $f_S(x)$ is
evaluated by using $2^{|S|}$ different values of $f(x)$. For every $x$, $f(x)
\in \{-1, 1\}$, and hence if $f_S(x) = 0$, it must be that the $2^{|S|}$ values
used in the expectation have exactly $2^{|S|-1}$ $+1$s and $-1$s each.

On the other hand, if $f_S(x) \neq 0$, then the number of $+1$s is different
than $-1$s. If $f_S(x) \neq 0$, without loss of generality, we only consider the
case when $f_S(x) > 0$, so that there are more $+1$s than $-1$s.  Thus, we are
left with the following combinatorial question:

Suppose we begin with $2k$ variables, $x_1, \ldots, x_{2k}$, where each $x_i$ is
$+1$ or $-1$. Let $k_1$ be the number of $+1$s and $2k - k_1$ is the number of
$-1$s. We will assume throughout that $k \geq 2$. We perform the following
process, each $x_i$ is left as is with probability $1 - \eta$ and its sign
flipped with probability $\eta$, independently. Let $x^\prime_i$ be the values
of the resulting variables, and let $X^\prime = \sum_i {x^\prime_i}$.  Let
$p^k_i$ denote the probability that $X^\prime$ is $0$ having started with
$(k+i)$ $+1$s and $(k -i)$ $-1$s.  Thus, $p^k_0$ is the probability of getting a
$0$, when we start with equal number of $+1$s and $-1$s. 
 
Then the following are true: 
\begin{enumerate}
\item $p^k_i$ decreases as $i$ increases.
\item $p^k_0 - p^k_1 \geq (2\eta - 1)^2 c_0 /k^{3/2}$ for some absolute constant
$c_0$.
\end{enumerate}
For proof of the above facts see Lemmas~\ref{lem:rcn1} and \ref{lem:rcn2} below, though it
should be fairly clear that the conclusions make sense. When $\eta = 1/2$, the
initial values are irrelevant of the $x_i$ are irrelevant and each $x^\prime_i =
\pm1$ with probability $1/2$, but for $\eta < 1/2$, if one started with the sum
$\sum_i {x_i} =0$, it is more likely that $\sum_i x^\prime_i =0 $, than if one
started from some value, $\sum_i x_i$ that was greater than $0$. 

We apply the above to the setting when $k = 2^{|S| - 1}$. We drop the
superscripts $p^{2^{|S|-1}}_0$ and $p^{2^{|S|-1}}_1$ in the rest of this
discussion. First, imagine that we have fixed the variables $x_{-S}$ so that
the expectation~(\ref{eqn:compute-fseta}) is only a function of the noise
function $\zeta$. If $f_S(x_{-S}) = 0$, then
$\Pr_{\zeta}[f^{\eta}_S(x_{-S}) = 0] = p_0$. On the other hand, if $f_S(x_{-S})
\neq 0$, then $0 \leq \Pr_{\zeta}[f^{\eta}(x_{-S}) = 0] \leq p_1$.
So, we have the following:
\begin{align*}
\Pr_{x, \zeta}[f^{\eta}_S(x) \neq 0] &\geq \Pr_{x}[f_S(x) \neq 0](1 - p_1) +
\Pr_{x}[f_S(x) = 0](1 - p_0) \\
&=(1 - p_0) + (p_0 - p_1) \Pr_{x \sim D}[f_S(x) \neq 0] 
\intertext{On the other hand,}
\Pr_{x, \zeta}[f^{\eta}_S(x) \neq 0] &\leq \Pr_{x}[f_S(x) \neq 0] + (1 - p_0)
\Pr_{x}[f_S(x) \neq 0] \\
&\leq (1-p_0) + \Pr_{x}[f_S(x) \neq 0]
\end{align*}
This completes the proof of the assertion. 
\end{proof}

We note that this allows us to distinguish between the cases where $\Pr_{x \sim
D}[f_S(x) \neq 0] \geq \alpha$ from $\Pr_{x \sim D}[f_S(x) \neq 0] \leq \beta$,
as long as $\alpha - \beta$ is sufficiently large. This can be done by choosing
$\beta = \alpha\cdot(2\eta -1)^2c_0/(2 \cdot 2^{3|S|/2})$, and then computing
the value $\Pr_{x \sim D}[f^{\eta}_S(x) \neq 0]$. Note that $p_0$ can be
computed exactly, if the size $|S|$ and the noise rate $\eta$ are known. We
assume that the noise rate is known; if not, the standard trick of binary
searching the noise rate can be employed. Note that these tests can be carried
out to high accuracy from samples. Now, in the case when $D$ is an
locally $\alpha$-smooth distribution for constant $\alpha$, any two points $x$ and
$x^\prime$ drawn from $\EX(f, D)$ will have Hamming distance $\Omega(n)$ with
very high probability.  The local queries to $\MQ(f^{\eta})$ are only made for
points that are at Hamming distance $O(\log(n))$ from sampled points (see
Fact~\ref{fact:smooth}). Thus, with very high probability, the queries made to
compute $f^{\eta}_S(x)$ and $f^{\eta}_S(x^\prime)$ do not have any point in
common, \ie no example is queried twice by the learning algorithm. So we can
employ Lemma~\ref{lem:rcn_separation} as if the noise was chosen independently
each time a point was queried. 

\subsubsection{$L_2$ Test}

Recall that $f^{\eta}_S(x_{-S}) = \frac {1}{2^{|S|}}\sum_{x_S \in
\moo^{|S|}}[f^{\eta}(x_Sx_{-S})]$. For a fixed $x_{-S}$, $f^{\eta}(x_{-S})$ is a
random variable depending only on the noise function $\zeta$. Let
$2^{|S|}f_{S}(x) = 2k$, where $2k$ is some even integer in the range $[-2^{|S|},
2^{|S|}]$. Let $k_1 = 2^{|S|-1} + k = 2^{|S|-1}(1+f_S(x))$ and $k_2 = 2^{|S|-1}
- k = 2^{|S|-1}(1 - f_S(x))$, so that $2^{|S|}f_S(x)$ is a sum of $k_1$, $+1$s
  and $k_2$, $-1$s. Let $Z_1 \sim \binomial(k_1, \eta)$ and $Z_2 \sim
  \binomial(k_2, \eta)$ be binomial random variables. Then
  $2^{|S|}f^{\eta}(x_{-S}) = 2^{|S|}f_S(x) - 2Z_1 + 2Z_2$. This follows
  immediately from the definition of the noise model. The following can then be
  verified by straightforward calculations,
\begin{align*}
\E_{\zeta}[f^{\eta}_S(x_{-S})] &= (1 - 2\eta) f_S(x) \\
\E_{\zeta}[f^{\eta}_S(x_{-S})^2] &= (1 - 2\eta)^2 f_S(x)^2 + 2^{-|S|+1}\eta (1 -
\eta)
\end{align*}
Thus, if we can obtain accurate estimates of $\E_{x \sim D}[f^{\eta}_S(x)^2]$,
we can also obtain accurate estimates of $\E_{x \sim D}[f_S(x)^2]$. Again, as in
the previous case, we observe that the algorithm (with high probability) never
makes a query twice for the same example. Thus, we can assume that the noise
model is in fact not persistent. It is clear that $\E_{x \sim
D}[f^{\eta}_S(x)^2]$ can be estimated highly accurately by sampling.

%% file: rcn_app.tex
The proof of the following two lemmas are elementary and are omitted.

\begin{lemma} \label{lem:randwalk} Suppose, $X_0 = 0$. Consider the following random walk, $X_{i+1} =
X_{i}$ with probability $1 - \alpha$, $X_{i+1} = X_i + 2$, with probability
$\alpha/2$ and $X_{i+1} = X_i - 2$, with probability $\alpha/2$, where $\alpha
\in [0, 1/2]$.  Then, for $i \geq 0$, $\Pr[X_n = 0] - \Pr[X_n = 2]$ is a
decreasing function of $\alpha$.  \end{lemma}

The idea of the proof is to notice that the probability, $\Pr[X_n = 2j]$ follows
a bell shaped curve, and the curve gets steeper (more mass is concentrated at
$0$) as $\alpha$ goes to $0$.
%
%\begin{proof}
%We prove this by induction on $n$. This is trivially true, when $n = 0$. Now,
%suppose the result is true for $n$. 
%
%$\Pr[X_{n+1} = i] = \Pr[X_n = i-1](\alpha/2) + \Pr[X_n = i+1](\alpha/2) +
%\Pr[X_n = i](1 - \alpha)$. Then, $\Pr[X_{n+1}=i] - \Pr[X_{n+1} = i+1] =
%(\Pr[X_{n} = i] - \Pr[X_{n} = i+1])(1 - 3\alpha/2) + (\Pr[X_n = i-1] - \Pr[X_n =
%i+2])(\alpha/2)$.
%
%\end{proof}

\begin{lemma} \label{lem:rcn1}
Let $x_1, \ldots, x_{2n}$, be such that $x_1 = \cdots = x_{n+d} = 1$ and
$x_{n+d+1} = x_{n+i+2} = \cdots x_{2n} = -1$. The sign of each $x_i$ is flipped
independently with probability $\eta < 1/2$, to get $x^\prime_i$. Let $p^n_d$ be
the probability that the $\sum_{i}x^\prime_i = 0$.  Then for $d \geq 0$, as $d$
increases, $p^n_d$ decreases.
\end{lemma}

This expresses the quite obvious idea that if the probability of flipping is
less than half, then the further from $0$ the initial sum $(\sum_{i} x_i)$, the
less likely it is that $\sum_{i} x^\prime_i = 0$.

\begin{lemma} \label{lem:rcn2}
Let $x_1, \ldots, x_{2n}$, be such that $x_1 = \cdots x_{n+1} = 1$ and $x_{n+2}
= \cdots = x_{2n} = -1$. The sign of each $x_i$ is flipped independently with
probability $\eta < 1/2$, to get $x^\prime_i$. Let $p_1$ denote the probability
that $\sum_{i} x^\prime_i = 0$. Let $y_1, \ldots, y_{2n}$ be such that, $y_1 =
\cdots y_n = 1$ and $y_{n+1} = \cdots = y_{2n} = -1$. Then, let $y^\prime_i$ be
obtained by flipping $y_i$ independently with probability $\eta < 1/2$, and let
$p_0$ denote the probability that $\sum_{i} y^\prime_i = 0$. Then $p_0 - p_1
\geq (2\eta - 1)^2 c_0/n^{3/2}$, for some absolute constant $c_0$.
\end{lemma}
\begin{proof}
First we leave aside the values, $x^\prime_n$, $x^\prime_{n+1}$,
$y^\prime_n$ and $y^\prime_{n+1}$. The remaining variables, both in the case of
$x_i$s and $y_i$s, were obtained by starting with exactly $(n-1)$ $+1$s and
$(n-1)$ $-1$s and flipping each independently with probability $\eta < 1/2$. We
can form pairs of $(+1, -1)$, to get a random variable $z_i = x_i^\prime +
x_{n+1+i}^\prime$, $i = 1, \ldots, n-1$, where $z_i = 0$ with probability
$\eta^2 + (1- \eta)^2 > 1/2$, $z_i = +2$ with probability $\eta(1-\eta)$ and
$z_i = -2$ with probability $\eta (1-\eta)$. (A similar argument can be made in
the case of $y^\prime_i$s.) We can view the sum of these $z_i$ random variables
as a random walk described in Lemma~\ref{lem:randwalk}, where $X_{i+1} = X_{i}$
with probability $\eta^2 + (1 - \eta)^2$ and $X_{i+1} = X_i + 2$, with
probability $\eta (1- \eta)$ and $X_{i+1} = X_i - 2$, with probability $\eta (1
- \eta)$.  Now, $p_1 = \Pr[X_{n-1} = 0](2 \eta (1-\eta)) + \Pr[X_{n-1} =
2]\eta^2 + \Pr[X_{n-1} = -2](1-\eta)^2$. On the other hand, $p_0 = \Pr[X_{n-1} =
0](\eta^2 + (1- \eta)^2) + \Pr[X_{n-1} = 2]\eta (1 -\eta) + \Pr[X_{n-1} = -2]
\eta (1 - \eta)$. Noticing that, $\Pr[X_{n-1} = 2] = \Pr[X_{n-1} = -2]$, we get
that $p_0 - p_1 = (2 \eta -1)^2 (\Pr[X_{n-1} = 0] - \Pr[X_{n-1} = 2])$. But this
difference is a decreasing function of $\alpha = 1 - (\eta^2 + (1 - \eta)^2)$.
But, even when $\alpha = 1/2$, i.e. $\eta = 1/2$, this difference is given by, 
\begin{eqnarray*}
\Pr[X_{n-1} = 0] - \Pr[X_{n-1} = 2] & = &\frac{1}{2^{2n-2}} \left( {2n-2 \choose
n-1} - {2n-2 \choose n} \right)\\
& = & \frac{1}{2^{2n}} {2n-2 \choose n-1}\left(1 -
\frac{n-2}{n} \right) 
\end{eqnarray*}
The claim now follows easily.
\end{proof}

%% file: lower_bound_app.tex
\begin{proof}[Proof of Theorem~\ref{thm:lower_bound}]
We begin by describing the required properties of the error correcting code.
For every integer $t$ and $m$ that is a power of two, there exists a binary BCH
code that maps a binary string $x$ of length $m-1-(t-1)\log m$ to a binary
string $z = x \cdot e(x)$ of length $m$ and has distance of $2t$. In particular,
if we denote the length of message $x$ by $n$ then for any $k$ we can obtain a
code with a codeword of length $m \leq n + k \log n$ and distance $2k+1$.

Given a function $f: \{-1,1\}^n \rightarrow \{-1,1\}$ we define a function $f_e:
\{-1,1\}^m \rightarrow \{-1,0,1\}$ as follows: $f_e(z) = f(x)$ if $z = x \cdot
e(x)$ for some $x \in \{-1,1\}^n$ and $f_e(z) = 0$ otherwise. Here the value $0$
is interpreted as function being equal to a random and independent coin flip.

We first note the following properties of this embedding.
\begin{itemize}
\item For every function $g: \{-1,1\}^n \rightarrow \{-1,1\}$,
$$\E_{x \cdot y \sim U_m}[f_e(x \cdot y) g(x) ] = 2^{n-m} \E_{x \sim U}[f(x) g(x) ]\ .$$
\item For every function $h: \{-1,1\}^m \rightarrow \{-1,1\}$,
$$\E_{z \sim U_m}[f_e(z) h(z) ] = 2^{n-m} \E_{x \sim U_n}[f(x) h(x \cdot e(x)) ]\ .$$
\end{itemize}

We can now describe the reduction from agnostic learning with $k$-local $\MQ$s to learning from random examples alone.
Let $C$ be a concept class, let $f$ denote an unknown target function and $\eps$ denote the error parameter.

The main idea is to simulate random examples and $k$-local queries to $f_e$ using random examples alone of $f$. The simulation requires observing that points in
$\{-1,1\}^m$ can be split into a set $Z$ which includes all codewords together with Hamming balls of radius $t$ around them and the rest of points (which we denote by $\bar{Z}$). We simulate a random example of $f_e(x)$ as follows:
\begin{enumerate}
\item Flip a coin with probability of heads equal to $\beta = |Z|/2^m$.
\item If the outcome is 1: ask for a random example and denote it by $(x,\ell)$. Choose a random point $z'$ in the Hamming ball of radius $k$ around $x \cdot e(x)$. If $z' = x \cdot e(x)$ then return the example $(z',\ell)$ otherwise return $(z',b)$ where $b$ is a random coin flip.
\item If the outcome is 0: sample a random point in $\bar{Z}$ and output $(z,b)$ where $b$ is a random coin flip. One can sample randomly from $\bar{Z}$ as follows: sample a random point in $z$ in $\{-1,1\}^m$, use a decoding algorithm for the BCH code to obtain a message $x$. If the decoding algorithm failed, that is $x \cdot e(x)$ is not within distance $k$ of $z$ then return $z$. Otherwise, try again.
\end{enumerate}
It is important to note that the expected number of tries of this algorithm is
$2^m/|\bar{Z}| = 2^m/(2^m-2^n\beta_{m,k})= 1/(1-2^{n-m}\beta_{m,k})$, where
$\beta_{m,k}$ denotes the size of the Hamming ball of radius $k$. We can always
assume that $2^{n-m}\beta_{m,k} \leq 2/3$ by for example adding 1 to $m$ (this
does not affect the code, increases  the term $2^{m-n}$ by $2$ and $\beta_{m,k}$
by at most $(1+k/(m+1))$). Therefore, with high probability, the simulation step
will not require more than a logarithmic number of tries. Note that the BCH code
we chose is efficiently decodable from up to $k$ errors~\citep{Massey:69}.

Now we simulate a $k$-local $\MQ$ $z$ as follows. If $z$ is $k$-close to random
example we generated in $\bar{Z}$ we return a random coin flip since there are
no non-zero values of $f_e$ within distance $k$ of any point in $\bar{Z}$. If
$z$ is $k$-close to random example we generated in $Z$ then it can only be
$k$-close to $x \cdot e(x)$ in the Hamming ball of which $z$ was generated and
for which we have an example $(x,\ell)$. This means that we can easily answer
this $\MQ$: if $z = x \cdot e(x)$ then we return label $\ell$, otherwise a random coin flip.

On these simulated examples we run the agnostic learning algorithm $\A$ for $C$ on $\{-1,1\}^m$  with $\eps' = \eps \cdot 2^{n-m} \geq \eps/n^k$. Let $h$ denote the hypothesis returned by the algorithm.

Let $c^* = argmax_{c \in C}\{\E_U[f(x) c(x)]\}$ and let $\Delta = \E_U[f(x) c^*(x)]$.
Then we know that $$\E_{x \cdot y \sim U_m}[f_e(x \cdot y) c^*(x) ] = 2^{n-m} \Delta .$$ By the agnostic guarantee of $\A$, we know that
$$\E_{z \sim U_m}[f_e(z) h(z) ] \geq 2^{n-m} \Delta + \eps' = 2^{n-m} (\Delta + \eps)\ .$$
Now, again by the properties of the embedding,
$$\E_{x \sim U_n}[f(x) h(x \cdot e(x)) ] = 2^{m-n} \E_{z \sim U_m}[f_e(z) h(z) ] \geq \Delta + \eps\ .$$
Hence $h'(x) = h(x \cdot e(x))$ is a valid hypothesis for agnostic learning of $C$.

Finally the running time of this simulation is $poly(n) \cdot T(n,n^k/\eps)$ where $T(n,n^k/\eps)$ is the running time of $\A$. In particular, if $T$ is polynomial then for a constant $k$ this simulation takes polynomial time.
\end{proof}

%% file: PAC_separate.tex
In this section, we show that $\PAC$+$r$-local $\MQ$ model is strictly more
powerful than the $\PAC$ model, assuming that one-way functions exist. In the following discussion we show that even
$1$-local membership queries are more powerful than the standard $\PAC$ setting.
%We note that it is known that the class of $k$-juntas is known to be learnable
%in $\poly(n, 2^k)$ time with $1$-local membership query. 

In this section, we assume that we are working with the domain $\{0,1\}^n$,
rather than $\moo^n$. Let $\mc{F}_n = \{f_s : \zo^n \rightarrow \zo \}_{s \in
\{0, 1\}^n}$ be a pseudo-random family of functions. It is well-known that such
families can be constructed under the assumptions that one-way functions
exist~\citep{GGM86}. Let $A_1, \ldots, A_n$ be a partition of
$\zo^n$ that is easily computable. For example, if the strings in $\zo^n$ are
lexicographically ordered, then $A_i$ contains strings with rank in the range
$[(i-1)2^n/n, i2^n/n)$. For an $n+1$ bit string $x$, $x_{-1}$ denotes the
$n$-length suffix of $x$. Then, for some string $s$, define the function $g_s
: \zo^{n+1} \rightarrow \zo$ as follows:
\[ g_s(x) = \begin{cases} f_s(x_{-1}) & \mbox{If $x_1 = 0$} \\ 
f_s(x_{-1}) \oplus s_i & \mbox{If $x_1 =1 $ and $x_{-1} \in A_i$} \end{cases} \]
Define $\mc{G}_{n+1} = \{ g_s : \zo^{n+1} \rightarrow \zo \}_{s \in \zo^n}$. We
show below that the class $\mc{G}_{n+1}$ is not learnable in the $\PAC$ setting,
but is learnable in the $\PAC$+$1$-local $\MQ$ model under the uniform
distribution.

\begin{theorem} Assuming that one-way functions exist, the class $\mc{G}_{n+1}$
is not learnable in the $\PAC$ model, but is learnable in the $\PAC$+$1$-local
$\MQ$ model, under the uniform distribution.
\end{theorem}
%
% UNCOMMENT FOR LONG VERSION
 \begin{proof}
 First, we show that $\mc{G}_{n+1}$ is learnable in the $\PAC$+$1$-local $\MQ$
 model.  Let $1x$ and $0x$ be the two strings of length $n+1$, with suffix $x \in
 \zo^n$.  Then for any $g_s \in \mc{G}_{n+1}$, $g_s(1x) \oplus g_s(0x) = s_i$, if
 $x \in A_i$. Thus, drawing a random example from $U$ and making a one local
 query reveals one bit of the string $s$. By drawing $O(n \log(n))$ random
 examples, all the bits of the string $s$ can be recovered with high probability.
 Thus, revealing the function $g_s$ itself. 
  
 On the other hand, in the $\PAC$ model, the probability that seeing two examples
 $1x$ and $0x$ is exponentially small.  Thus, all the labels appear perfectly
 random (since $f_s$ is from a pseudorandom family). Thus, no learning is
 possible in the $\PAC$ model.
 \end{proof}

In fact, the above construction also shows that the random walk learning model~\citep{Bshouty05} is also more powerful than the $\PAC$ learning setting,
assuming that one way function exist. \citet{Bshouty05} had already
shown that the random walk model is provably weaker than the full $\MQ$ model
assuming that one-way functions exist. In fact, essentially the same argument also
shows that full $\MQ$ is more powerful than $\PAC$+$o(n)$-local $\MQ$. The
following simple concept class (which is the same as that of Bshouty et al.)
shows the necessary separation.

Let $e^i$ be the vector that has $1$ in the $i\th$ position, and $0$s elsewhere.
Again, let $\mc{F}_n = \{f_s : \zo^n \rightarrow \zo \}_{s \in \zo^n}$ be the
pseudorandom family of functions. Then define, $\mc{G}^\prime_n = \{g_s \}$ as
follows: \[ g_s(x) = \begin{cases} s_i & \mbox{If $x = e^i$} \\ f_s(x) &
\mbox{Otherwise} \end{cases} \]

\begin{theorem} The concept class $\mc{G}^\prime_n$ is learnable in the full
$\MQ$ model, but not in $\PAC$+$o(n)$-local $\MQ$ model under the uniform
distribution.  \end{theorem}
%% UNCOMMENT FOR LONG VERSION
 \begin{proof}
 It is easy to see that by making membership queries to the points, $e^1, \ldots,
 e^n$, the string $s$ is revealed and hence also the function $g_s$. On the other
 hand, random points from the Boolean cube have Hamming weight $\Omega(n)$,
 except with exponentially small probability. Thus, $o(n)$-local $\MQ$s are of no
 use to query the points $e^i$. The labels for any point obtained from the
 distribution, or using $o(n)$-local $\MQ$s are essentially random. Hence,
 $\mc{G}^\prime_n$ is not learnable in the $\PAC$+$o(n)$-local $\MQ$ model.
 \end{proof}